\documentclass[sn-mathphys-num]{sn-jnl}% Math and Physical Sciences Numbered Reference Style 
\usepackage{multirow}%
\usepackage{amsmath,amssymb,amsfonts}%
\usepackage{amsthm}%
\usepackage{mathrsfs}%
\usepackage[title]{appendix}%
\usepackage{xcolor}%
\usepackage{textcomp}%
\usepackage{manyfoot}%
\usepackage{booktabs}%
\usepackage{algorithm}%
\usepackage{algorithmicx}%
\usepackage{algpseudocode}%
\usepackage{listings}%
\usepackage{latexsym}
\usepackage{enumitem}
\usepackage{graphicx}
\usepackage{xcolor}

\usepackage[section]{placeins}
\usepackage{tabularx}  % ????????

\theoremstyle{thmstyleone}%
\newtheorem{theorem}{Theorem}%  meant for continuous numbers
\theoremstyle{thmstyletwo}%

\theoremstyle{thmstylethree}%

\begin{document}

\title[Article Title]{EMP: Enhance Memory in Data Pruning}


\author[1,2]{\fnm{Jinying} \sur{Xiao}}\email{jinyingxiao@nudt.edu.cn}
\equalcont{These authors contributed equally to this work.}

\author*[2,3]{\fnm{Ping} \sur{Li}}\email{lping9188@163.com}
\equalcont{These authors contributed equally to this work.}

\author[2]{\fnm{Jie} \sur{Nie}}\email{csustniejie@163.com}

\author*[1]{\fnm{Bin} \sur{Ji}}\email{ jibin@nudt.edu.cn}
\author[1]{\fnm{Shasha} \sur{Li}}\email{  shashali@nudt.edu.cn}
\author[1]{\fnm{Xiaodong} \sur{Liu}}\email{  liuxiaodong@nudt.edu.cn}
\author[1]{\fnm{Jun} \sur{Ma}}\email{  majun@nudt.edu.cn}
\author[1]{\fnm{Qingbo} \sur{Wu}}\email{  wuqingbo@kylinos.cn }
\author*[1]{\fnm{Jie} \sur{Yu}}\email{  yj@nudt.edu.cn}

\affil*[1]{\orgdiv{College of Computer}, \orgname{National University of Defense Technology}, \orgaddress{\street{Deya Road}, \city{Changsha}, \postcode{410073}, \state{Hunan}, \country{China}}}

\affil*[2]{\orgdiv{School of Computer and Communication Engineering}, \orgname{Changsha University of Science and Technology}, \orgaddress{\street{Chiling Road}, \city{Changsha}, \postcode{410076}, \state{Hunan}, \country{China}}}

\affil[3]{\orgdiv{Hunan Provincial Key Laboratory of Intelligent Processing of Big Data on Transp}, \orgaddress{ \city{Changsha}, \postcode{410114}, \state{Hunan}, \country{China}}}


\abstract{Large language and vision models have demonstrated remarkable performance, but their high pre-training and fine-tuning costs have led to growing interest in accelerating training through dataset pruning. Traditional pruning approaches typically rely on sample loss to identify and retain the most "difficult" samples. However, as the pruning rate increases, the training frequency of each sample becomes more uniform, which can result in the underfitting of critical or general samples. We identify this phenomenon as Low-Frequency Learning (LFL), where the model fails to retain knowledge of the majority of samples. In this work, we decompose the scoring function of LFL, providing a theoretical analysis of its inefficiencies. To counteract this issue, we propose a novel enhancement to the scoring function by introducing a memory term designed to strengthen the model's ability to memorize essential data. We also offer an approximation for this memory term. Additionally, we extend this concept to Self-Supervised Learning (SSL), marking the first investigation into the role of memory in SSL. By leveraging contrastive learning, we derive the memory term both theoretically and experimentally. Based on these insights, we introduce Enhance Memory Pruning (EMP), a technique that mitigates memory loss under high pruning rates by improving the model's data retention. EMP has been evaluated across various tasks, including image classification, natural language understanding, and model pre-training. Our experiments demonstrate that EMP significantly enhances model performance, particularly under extreme pruning conditions. For instance, in the CIFAR100-ResNet50 pre-training task with a 70\% pruning rate, EMP outperforms current state-of-the-art methods by 2.2\%.\footnote{The code was shown in https://github.com/xiaojinying/EMP}}




\maketitle
\section{Introduction}
Recently, deep learning technology has achieved significant breakthroughs in fields such as visual recognition \cite{ref20,ref21,ref22} and natural language processing \cite{ref23,ref24,ref25}. Although these models exhibit exceptional performance, they often require training and fine-tuning on large datasets, especially in the pre-training of large language models (LLMs). This process not only demands substantial computational resources but also consumes a considerable amount of time. Therefore, reducing the burden of training models on large-scale datasets has become increasingly important for promoting the application of deep learning technology across a broader range of fields.

Data pruning accelerates model pre-training by retaining a core subset of typical or general samples, thereby reducing the number of training iterations. Recently, dynamic pruning methods that adjust the retained dataset in real-time have become popular \cite{ref4,ref5,ref27}. These methods dynamically score data at each checkpoint, allowing pruned data to be retrained. Since these methods do not require additional training of a proxy network, they have started to gain popularity. Some of these methods use sample loss to construct \cite{ref4,ref5}, selecting the most "difficult" samples to train at each checkpoint. \cite{ref5} indicates that such methods can acquire more information and knowledge. However, we found that these methods are only effective at specific pruning rates. In other words, when the pruning rate is high, these methods exhibit poor performance. Moreover, on datasets that are difficult to fit, such as ImageNet-1K, the performance of low-frequency learning is comparable to, or even lower than, dynamic random pruning methods (see Table~\ref{table2}).

Naturally, we explored the reasons behind this. It is worth noting that these methods retain samples with the highest loss at each checkpoint. However, during the training process, due to the model's strong fitting ability \cite{ref28}, the loss of these samples will eventually decrease \cite{ref26}, and these samples are often pruned in the next selection. This results in an even distribution of the number of times samples are selected, which we refer to as Low-Frequency Learning (LFL). As shown in Figure~(\ref{fig1}), we counted the number of times each sample was selected during the entire training process and found that in LFL, the selection frequency of samples is more concentrated. In other words, the difference in the number of times each sample is included in the training is small, and the number of times is generally low. This limits the number of times each sample is trained. Previous work has shown that repeated learning can enhance model memory \cite{ref18}; for instance, in language models, the model tends to capture familiar phrases and commonly accepted knowledge \cite{ref8}.

\begin{figure}[h]
	\centering
	\includegraphics[width=0.8\textwidth]{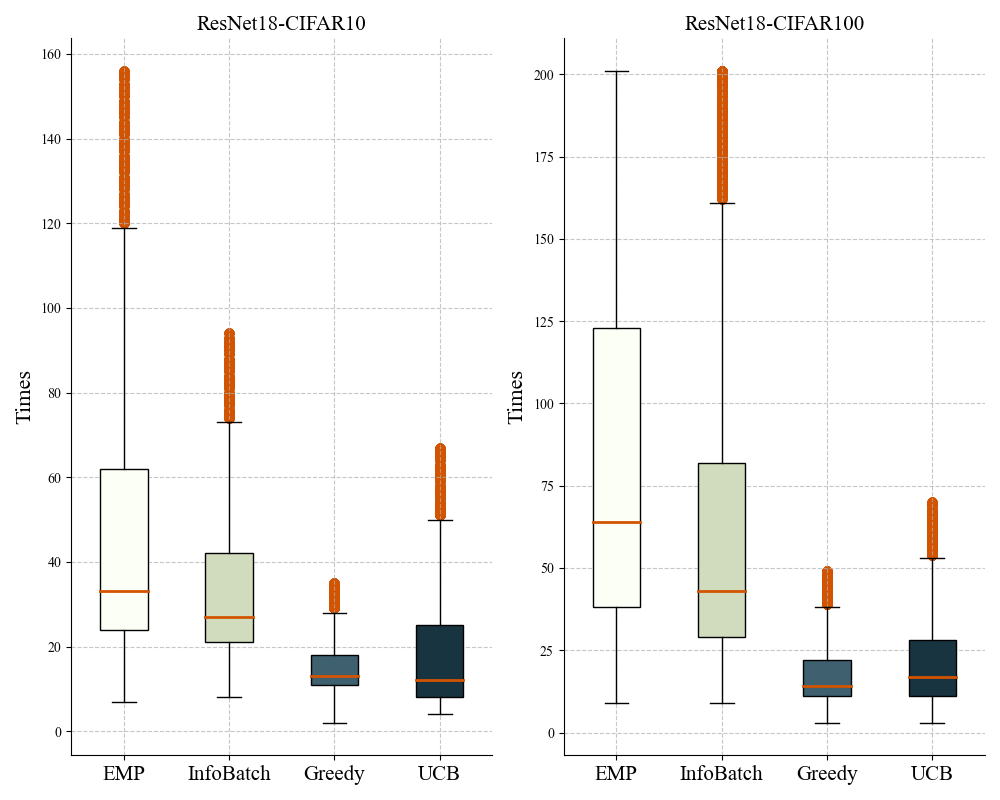}
	\caption{Throughout the entire training process (200 epochs), the number of times each sample is selected is collected under different data pruning algorithms at a pruning rate of 90\%. Among them, InfoBatch, Greedy, and UCB are all pruning methods that score based on sample loss, which is known as Low-Frequency Learning (LFL).}
	\label{fig1}
\end{figure}

Our contributions are as follows:
\begin{itemize} 
	\item Through theoretical analysis, we have studied the inefficiency of LFL and its causes. In this work, we decomposed the cross-entropy function, extracting the memory term that is negatively correlated with cross-entropy. By analyzing the dynamic relationship between LFL and this memory term, we theoretically concluded that LFL leads to insufficient model memory of samples. On the contrary, learning multiple times from typical or general samples is our main research objective.
	\item Our research focused on the memory of our model and advocated for the addition of a memory term to the scoring function to enhance the model's memory of samples. In this work, we specifically discussed two scenarios: supervised learning (SL) and self-supervised learning (SSL). To our knowledge, this is the first time memory has been discussed in the context of SSL. In SL, we used the mutual information term extracted from the cross-entropy function as the memory term, which quantifies the distribution relationship between model weights and samples. Since the mutual information term involves complex distributions of data and parameters, it is not easy to calculate, and we approximated it. In SSL, we focused on contrastive learning (CL). Based on previous work on memory in SL, we transferred their ideas to CL. Since the model's memory of samples is specifically reflected in the hidden layers of the basic encoder, we established a memory term through this hidden layer. Our idea is illustrated in Figure~(\ref{fig2}).
	
	\item We propose the Enhance Memory Pruning (EMP) method, which is a dynamic data pruning approach. We examined the performance of EMP on image classification tasks, natural language understanding tasks, and model pretraining tasks. Through experimental validation, EMP has been shown to lead the state-of-the-art methods in most cases and has achieved commendable performance under high pruning rates. For instance, in image classification, with CIFAR100-ResNet18, even after pruning 70\% of the dataset, EMP still outperforms other methods by at least 2.1\%.
	
\end{itemize}
\begin{figure}[h]
	\centering
	\includegraphics[width=0.8\textwidth]{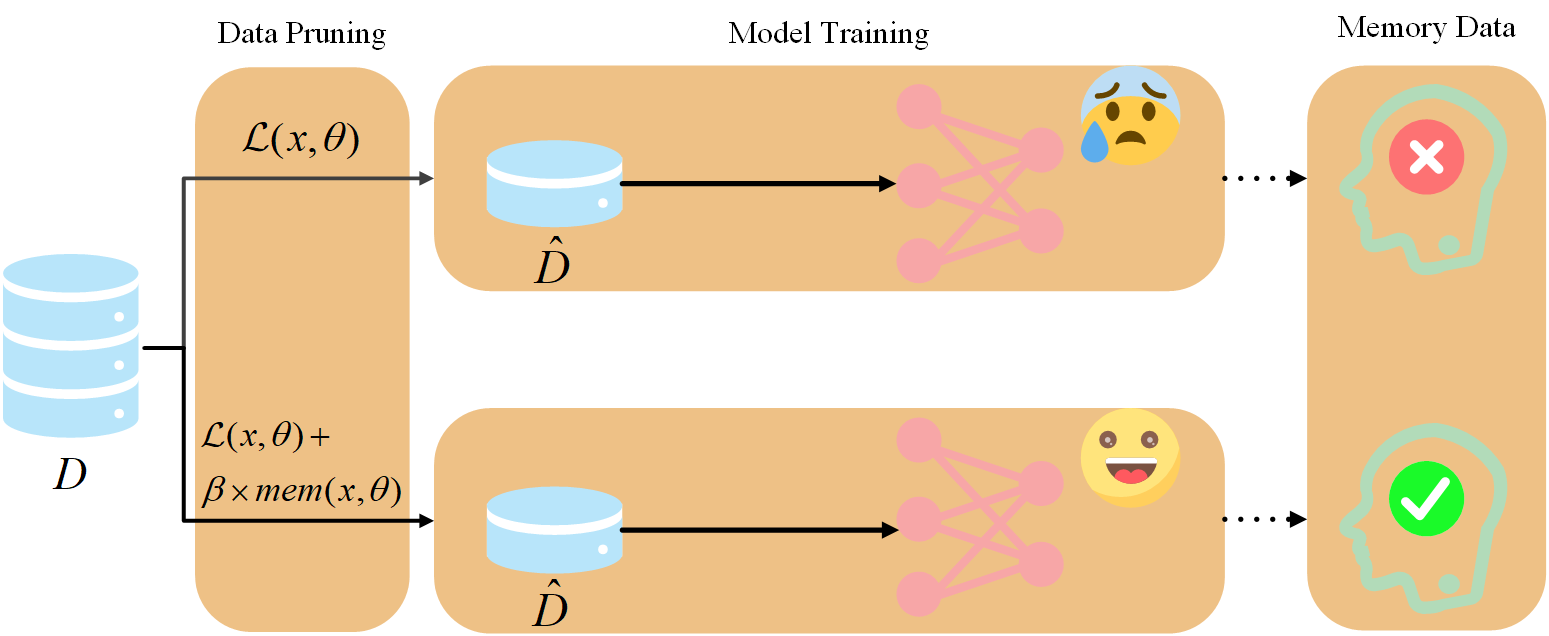}
	\caption{Unlike other methods that use sample loss for scoring, we enhance model memory by adding a memory term \( mem(x, \theta) \), where \( \beta \) is an adjustable hyperparameter.
	}
	\label{fig2}
\end{figure}

\section{Related Work}
\textbf{Static Pruning.} Static pruning aims to select a compact core subset before training. \cite{ref3} analyzed the forgetting of samples during the training process and assessed their forgetfulness in training to eliminate unforgettable samples, \cite{ref44} selected a moderate core set based on the distance of data points to the center, \cite{ref43} utilized the maximum diversity of samples in the gradient space for selection, \cite{ref33} ingeniously used the influence function, constructing a core subset based on the impact of samples on the model's generalization ability. \cite{ref17} designed a first-order gradient approximation to assess the impact of samples on the optimal empirical risk. Although these methods select an efficient core subset, the core subset, as a form of a posteriori knowledge, must be learned through the optimization of one or more proxy models to understand the characteristics and distribution of the data. For example, in \cite{ref33}, SENet and ResNet18 are used as proxy models to accelerate the training of ResNet50. Although these proxy models are smaller in scale, they require training on the entire dataset, which inevitably brings additional overhead. More importantly, these methods result in a core subset with poor generalizability. For instance, in \cite{ref33}, for different specifications of ResNet, the proxy model also requires additional selection and optimization.

\textbf{Dynamic Pruning.} To address the aforementioned additional overhead, \cite{ref5} argued that the optimal dynamic scoring is closely integrated with the training trajectory of the model. They categorized data into three types based on the number of times samples were selected and found that these samples are highly variable and can transition within the training dynamics. They first performed dataset pruning on the retained data's loss at each checkpoint during training without a proxy network, and from their conclusions, dynamic pruning is always superior to static pruning, even random selection. In \cite{ref4}, soft pruning was proposed; they believed that for a dataset with N samples, hard pruning with an unchanged pruning rate requires a complexity of \( \mathcal{O}(\log N) \) for sorting the scores, while soft pruning only requires \( \mathcal{O}(1) \). Although the sorting cost is reduced, experimentally, soft pruning cannot determine the true pruning ratio, thereby explicitly increasing the training cost of the model. It is worth noting that before us, \cite{ref4,ref5} provided state-of-the-art performance in dynamic pruning.

\section{Methodology}
\subsection{Problem Description}
Given a large-scale dataset \( D = \{X, Y\} \) containing \( n \) training samples, where the input \( X = \{x^{(1)}, \ldots, x^{(n)}\} \), labels \( Y = \{y^{(1)}, \ldots, y^{(n)}\} \), and \( f(Y | X, \theta) \) represents the network output parameterized by \( \theta \). The goal of the data is to identify a subset \( \hat{D} = \{\hat{X}, \hat{Y}\} \), where \( \hat{D} \subseteq D \), thereby accelerating the model training process. It is important to note that the pruning rate \( s \) is expressed as \( s = \frac{\Vert D - D_0 \Vert_0}{\Vert D_0 \Vert_0} \). In this paper, the data is generated by the distribution \( p(x, y) \), and we need to use several information quantities such as entropy: $ H(X) = -\mathbb{E} [\log p(x)] $, mutual information: \( I(X; Y) = H(X) + H(Y) - H(X, Y) \), and Kullback-Leibler divergence: \( KL(p(x) || q(x)) = \mathbb{E}_{x \sim p(x)} [\log(p(x) / q(x))] \).

\subsection{LFL Leads to Poor Memory}\label{3.2}
Previous works \cite{ref4,ref5} used loss values to score samples, and at each checkpoint, they selected the portion with the highest loss values according to the pruning rate. We believe this is a form of low frequency learning. Specifically, at the first checkpoint, they selected \( \hat{D}_1 \), and due to the model's gradient descent on \( L(\hat{D}_1, \theta) \) during the training process, these data are very likely not to be retained at the next checkpoint. On the contrary, the algorithm tends to retain the data in \( D_{-\hat{D}_1} \), at which point the chance of data being retained is averaged, as shown in Figure~(\ref{fig1}). However, repeated learning can strengthen the model's memory \cite{ref18}, for example, in language models, the model is more inclined to capture familiar phrases, recognized knowledge \cite{ref8}. Intuitively, LFL does not provide the model with the opportunity for repeated learning, leading to insufficient memory of the data by the model, especially under high pruning rates.

In our experiments, we tested the aforementioned viewpoint, as shown in Figure~(\ref{fig3}), where we established LFL under an extreme condition where each sample is trained the same number of times, which we refer to as Extreme Low-Frequency Learning (ELFL). It is not difficult to see from the graph that under high pruning rates, the model's training accuracy is not high. This situation is not unexpected, as it can be observed that in the early stages of training, due to the inconsistent gradient information from atypical and noisy examples, they may cancel each other out \cite{ref9,ref10}, at which point the model tends to learn general patterns \cite{ref18}, and the rise in training accuracy is not significant. In the later stages of training, the InfoBatch method, which uses loss as a criterion, and ELFL do not show a significant increase in training accuracy, which is particularly prominent in the challenging CIFAR100 dataset. Combining this with the long-tail theory \cite{ref11,ref19}, we believe that in the later stages, important or typical samples appear infrequently, and these samples require more repeated memorization, while the model only remembers those long-tail samples that are easier to remember. Therefore, based on the experiments, pruning methods that use loss as a criterion struggle to memorize the data.

\begin{figure}[h]
	\centering
	\includegraphics[width=\textwidth]{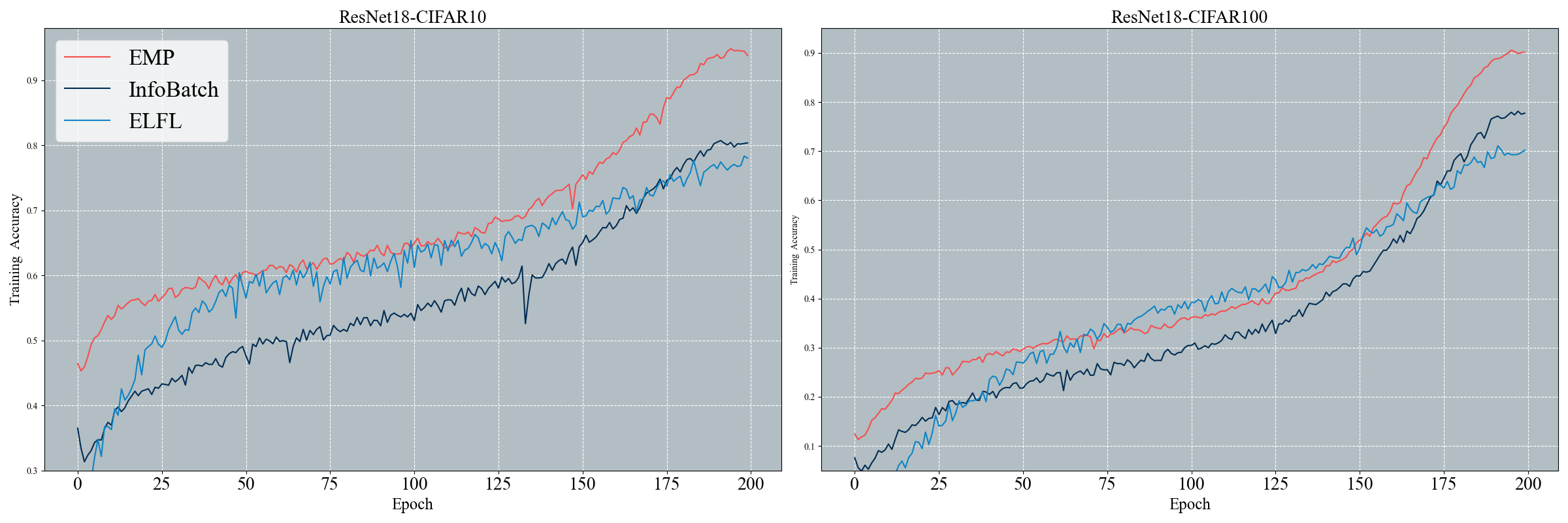}
	\caption{At a pruning rate of 90\%, the training loss across different algorithms and datasets is compared. Among them, ELFL represents Extreme Low-Frequency Learning, and InfoBatch is a method of LFL (Low-Frequency Learning).
	}
	\label{fig3}
\end{figure}

We theoretically explain why selecting samples with the highest loss at each checkpoint can lead to insufficient model memory, which is a phenomenon observed in practice.

In previous work \cite{ref1,ref2}, when the data is generated consistently, the training loss (scoring function) in the form of cross-entropy can be decomposed as:

\begin{eqnarray}
	\label{eq1}
	CE(D,\theta)=H(Y|X)-I(\theta;Y|X)+\mathbb{E}_{(X,\theta)}\ KL[p(Y|X)||f(Y|X,\theta)]
\end{eqnarray}

Where \( CE(D, \theta) \) represents the cross-entropy loss function.

However, it can be noted that, in order to obtain more information \cite{ref3}, previous data pruning methods \cite{ref3,ref4,ref5} select the samples with the highest training loss at each checkpoint. In fact, such algorithms can achieve this by minimizing the second term in Equation~(\ref{eq1}).

\newtheorem{lemma}[theorem]{Lemma}

%\newdefinition{rmk}{Remark}
%\newproof{pf}{Proof}
\begin{theorem}
	\label{th1}
	For a set of \( m \) samples, an independently and identically distributed subset of retained data \( (\hat{X}, \hat{Y}) = \{ (x^{(1)}, y^{(1)}), \ldots, (x^{(m)}, y^{(m)}) \} \), let \( \hat{y}^{(i)} \) represent the model's prediction for the \( i \)-th sample, and let \( c^{(i)} = \mathbf{1} \{ \hat{y}^{(i)} = y^{(i)} \} \) be the correctness variable corresponding to the prediction \( \hat{y}^{(i)} \). Thus, the following inequality holds:
	\begin{eqnarray}
		\label{eq2}
		\mathbb{E}\left[\sum_{i=1}^{m}c^{\left(i\right)}\right]\le\frac{log\left(\left|\hat{Y}\right|-1\right)-H\left(\hat{Y}\middle|\hat{X}\right)+I\left(\theta;\hat{Y}\middle|\hat{X}\right)+\sum_{i=1}^{m}H\left(1-c^{\left(i\right)}\right)}{log\left(\left|\hat{Y}\right|-1\right)}
	\end{eqnarray}
	
\end{theorem}
\begin{proof}
	
	For each sample \( (x_i, y_i) \), we consider the following Markov chain:
	
	$y^{\left(i\right)}\rightarrow\left[\begin{matrix}\hat{X}\\\hat{Y}\\\end{matrix}\right]\rightarrow\left[\begin{matrix}x^{\left(i\right)}\\\theta\\\end{matrix}\right]\rightarrow\hat{y}^{\left(i\right)}$
	
	Under this Markov chain, Fano's inequality \cite{ref7} provides a lower bound on the training error:
	\begin{eqnarray}
		\label{eq3}
		P\left(c^{\left(i\right)}=0\right)\geq\frac{H\left(y^{\left(i\right)}\middle| x^{\left(i\right)}\right)-H\left(1-c^{\left(i\right)}\right)}{log\left(\left|\hat{Y}\right|-1\right)}
	\end{eqnarray}
	
	For the retained dataset \( (\hat{X}, \hat{Y}) \) with a sample size of \( m \), the sum can be obtained as follows:
	\begin{eqnarray}\label{eq4}
		\begin{split}
			\sum_{i=1}^{m}P\left(c^{\left(i\right)}=0\right)\geq&\frac{\sum_{i=1}^{m}\left(H\left(y^{\left(i\right)}\middle|x^{\left(i\right)},\theta\right)-H\left(1-c^{\left(i\right)}\right)\right)}{log\left(\left|\hat{Y}\right|-1\right)}\\\geq&\frac{\sum_{i=1}^{m}\left(H\left(y^{\left(i\right)}\middle|\hat{X},\theta\right)-H\left(1-c^{\left(i\right)}\right)\right)}{log\left(\left|\hat{Y}\right|-1\right)}\\\geq&\frac{H\left(\hat{Y}\middle|\hat{X},\theta\right)-\sum_{i=1}^{m}H\left(1-c^{\left(i\right)}\right)}{log\left(\left|\hat{Y}\right|-1\right)}
		\end{split}
	\end{eqnarray}
	
	Given that \( H(Y | X, \theta) = H(Y | X) - I(\theta; Y | X) \), the result is obtained as:
	\begin{eqnarray}\label{eq5}
		\mathbb{E}\left[\sum_{i=1}^{m}c^{\left(i\right)}\right]\le\frac{log\left(\left|\hat{Y}\right|-1\right)-H\left(\hat{Y}\middle|\hat{X}\right)+I\left(\theta;\hat{Y}\middle|\hat{X}\right)+\sum_{i=1}^{m}H\left(1-c^{\left(i\right)}\right)}{log\left(\left|\hat{Y}\right|-1\right)}
	\end{eqnarray}
\end{proof}

Theorem~\ref{th1} establishes an expected upper bound on the training accuracy on the retained dataset, and it can be observed that this upper bound decreases as \( I(\theta; Y | X) \) decreases. As mentioned earlier, by selecting for learning based on the largest losses, such algorithms can achieve this by minimizing the second term in Equation~(\ref{eq1}). In this case, the \( I(\theta; \hat{Y} | \hat{X}) \) of the retained data tends to decrease, and according to Equation~(\ref{eq2}), we can intuitively consider that the training accuracy will decrease at this time.

\begin{lemma}
	\label{le1}
	In dynamic pruning, the data pruning algorithm \( A(D)_k \) based on loss selects the data for the \( k \)-th epoch, that is:
	
	$\forall z_i\in D,\forall z_j\in A\left(D\right)_k,\mathcal{L}\left(z_i,\theta_k\right)\le\mathcal{L}\left(z_j,\theta_k\right)$
	
	Then the resulting model has a weak memory of the data, in other words, the model fits the training data poorly.
\end{lemma}
\begin{proof}
	According to Equation~(\ref{eq1}), when fitting the model with the subset of data that has the highest training loss, algorithms with limited memory capacity for labels will cause \( I(\theta; \hat{Y} | \hat{X}) \) to tend towards 0. At this point, according to Theorem~\ref{th1}, this will reduce the upper bound of training accuracy, thereby affecting the model's fit.
	
	Lemma~\ref{le1} indicates that when fitting the model with the subset of data that has the highest training loss, it makes the model weights \( \theta \) and the data labels \( \hat{Y} \) become more independent, reducing their correlation. This result means that the model has difficulty remembering these data. In simple terms, these algorithms fail to fit correctly when faced with data with large losses, and the phenomenon highlighting this situation is the insufficiency of training accuracy.
\end{proof}

\subsection{Enhancing Memory in SL}\label{3.3}
In Section~\ref{3.2}, the inability of LFL to enable model memorization of data was analyzed. In the context of supervised learning, looking at Equation~(\ref{eq1}), we need to increase \( I(\theta; Y | X) \) to enable the model to remember the data. To achieve this goal, it is first necessary to estimate \( I(\theta; Y | X) \), which involves the independence of \( \theta \) and \( Y \). Previous work has used gradients \cite{ref1} and differential methods \cite{ref12} to estimate mutual information. In our work, we decompose \( I(\theta; Y | X) \) as follows:
\begin{eqnarray}
	\label{eq6}
	I\left(\theta;Y\middle| X\right)=H\left(\theta\right)+H\left(Y\middle| X\right)-H\left(\theta,Y\middle| X\right)
\end{eqnarray}

Generally, \( H(\theta, Y | X) \leq H(\theta) \), and our proof can be found in the Appendix~\ref{app_1}. Thus, we can establish a lower bound for \( I(\theta; Y | X) \):
\begin{eqnarray}
	\label{eq7}
	I\left(\theta;Y\middle| X\right)\geq H\left(Y\middle| X\right)
\end{eqnarray}

Where the conditional entropy \( H(Y | X) \) is a measure of how uncertain Y is given X. For a sample \( (x^{(i)}, y^{(i)}) \), we represent the conditional entropy of a sample by \( H(f(x^{(i)}, \theta)) \), where \( f(x^{(i)}, \theta) \) denotes the output of the model with parameters \( \theta \) for the input \( x^{(i)} \).

Therefore, we establish the scoring function for the supervised learning scenario. For the data \( (x^{(i)}, y^{(i)}) \), we use the following formula to score it:
\begin{eqnarray}
	\label{eq8}
	score\left(\left(x^{\left(i\right)},y^{\left(i\right)}\right)\right)=\mathcal{L}\left(\left(x^{\left(i\right)},y^{\left(i\right)}\right),\theta_k\right)+\beta H\left(f\left(x^{\left(i\right)},\theta_k\right)\right)
\end{eqnarray}

In which, \( \theta_k \) represents the current model parameters, \( \mathcal{L}(\cdot) \) generally refers to the cross-entropy function, and \( \beta \) is an adjustable hyperparameter that balances the model's learning of general patterns and memorization of data. The core idea of this scoring function is to use the sample loss to enable the model to obtain as much information as possible, allowing the model to learn not only general patterns but also, through the second term, to prefer samples that are easy to remember, achieving higher model performance.

\begin{lemma}
	\label{le2}
	For a subset \( \hat{D} \subseteq D \), if \( \hat{D} \) satisfies:
	
	$\forall\left(x^{\left(i\right)},y^{\left(i\right)}\right)\in\hat{D},\forall\left(x^{\left(j\right)},y^{\left(j\right)}\right)\in D-\hat{D},score\left(\left(x^{\left(i\right)},y^{\left(i\right)}\right)\right)>score\left(\left(x^{\left(j\right)},y^{\left(j\right)}\right)\right)$
	
	Then, compared to LFL, the subset \( \hat{D} \) will result in a model with a smaller upper bound on generalization error.
\end{lemma}

\begin{proof}
	
	We assume that the subset generated by LFL is \( \hat{D}_{\text{LFL}} \). Previous work \cite{ref29} has expressed the upper bound on generalization error, which in our context can be represented as:
	\begin{eqnarray}
		\label{eq9}
		\mathcal{R}\left({\hat{\theta}}_{\hat{D}}\right)\le\hat{\mathcal{R}}\left(\hat{D},{\hat{\theta}}_{\hat{D}}\right)+\varepsilon
	\end{eqnarray}
	
	Where \( \mathcal{R}(\hat{\theta}_{\hat{D}}) \) is the expected loss, \( \hat{\mathcal{R}}(\hat{D}, \hat{\theta}_{\hat{D}}) \) is the empirical risk on \( \hat{D} \), which is the fitting loss, \( \varepsilon \) is a coefficient related to the model size and the size of the retained dataset, which remains fixed after the pruning rate and model architecture are set, and \( \hat{\theta}_{\hat{D}} \) represents the optimal parameters obtained after optimization with \( \hat{D} \).
	Our method enhances the model's memory, and likewise, improves the model's fitting capability. Therefore, based on the above experimental and theoretical analysis, we can conclude that:
	\begin{eqnarray}
		\label{eq10}
		\hat{\mathcal{R}}\left(\hat{D},{\hat{\theta}}_{\hat{D}}\right)<\hat{\mathcal{R}}\left({\hat{D}}_{\text{LFL}},{\hat{\theta}}_{\hat{D}_{\text{LFL}}}\right)
	\end{eqnarray}
	
	Thus, Lemma~\ref{le2} is proven. Compared to LFL that uses sample loss, our scoring function results in a model with a smaller upper bound on generalization error.
\end{proof}
The data pruning method in SL is presented in Algorithm~\ref{alorithm 1}.
\begin{algorithm}[h]
	\caption{EMP in SL}
	\label{alorithm 1}
	\begin{algorithmic}[1]
		% {\bfseries Output:}  $\mathbf{c}$
		\Require Dataset \( D \), Epoch \( T \), Pruning Ratio \( s \), Initial Model \( \theta \), Retained Dataset \( \hat{D}_0 = D \).
		
		\hrule
		\vspace{0.1cm}
		\For{$t$ from 0 to $T-1$}
		\If{$t!=0$}
		\State Select the retained dataset \( \hat{D}_t \) based on the pruning rate \( s \) and sample scores.
		
		\Else
		\State $\hat{D}_t=D$
		\EndIf
		\State Get train sequence $\hat{D}_t=\{B_0,B_1,,,B_{b_t-1}\}$
		\For{$j$ from 0 to $b_t-1$}
		\State Get loss $L(B_j,\theta_t^j)$
		\State Update score with Equation~(\ref{eq8})
		\State Update model with optimizer
		\EndFor
		%\vspace{0.1cm}
		%\hrule
		\EndFor
	\end{algorithmic}
\end{algorithm}

\subsection{Enhancing Memory in SSL}\label{3.4}
In SSL, we mainly discuss Contrastive Learning (CL). CL learns representations of data by focusing on positive samples that are close to each other (such as augmentations from the same image) and excluding negative samples (such as augmentations from different images) \cite{ref13}. By promoting the closeness of positive examples and maximizing the separation between negative examples in the latent space, it learns the representation of the data \cite{ref14}.

In CL, a relatively simple and practical framework is SimCLR \cite{ref15}, which is an easy-to-implement visual representation contrastive learning framework. We mainly discuss the training process of the basic encoder \( f(\cdot) : \mathbb{R}^{L \times 1} \longrightarrow \mathbb{R}^{H \times 1} \), where \( f(\cdot) \) typically uses a ResNet architecture \cite{ref16}. For a sample \( x \in \mathbb{R}^{L \times 1} \), two independent data augmentation operators can be sampled to obtain \( x_i, x_j \in \mathbb{R}^{L \times 1} \). The basic encoder maps the two data to intermediate representations \( f(x_i), f(x_j) \). During the training phase, a projection head \( g(\cdot) \) is often used to obtain the outputs \( g(f(x_i)), g(f(x_j)) \). And training is conducted using the NT-Xent loss to decrease.

To our knowledge, this is the first time discussing data memory in CL. As discussed in Section~\ref{3.2}, under high pruning rates, it is necessary to strengthen the model's memory of data to achieve more efficient performance. Since \( I(\theta; Y | X) \) is extracted from the cross-entropy function, and CL training is independent of data labels \( Y \), therefore, directly transferring the memory enhancement methods from SL to CL is not effective.

Previous work \cite{ref11} has explored model memory, where in SL, they compared the model's probability of a sample before and after the sample's participation in training. Since in CL, the model primarily learns the intrinsic features of the data without involving data labels in training, we extend this concept to a more general form:

\begin{eqnarray}
	\label{eq11}
	mem\left(x\right)=\begin{matrix}loss\left(x_i,x_j\right)\\A\left(D-x\right)\\\end{matrix}-\begin{matrix}loss\left(x_i,x_j\right)\\A\left(D\right)\\\end{matrix}
\end{eqnarray}

In which, \( A(D) \) represents the optimization algorithm utilizing the entire dataset, and the construction of the \( \text{loss}() \) function is worth discussing.

It is worth noting that in CL, if the model fits the sample \( x \) well, it implies that the model will judge the two data generated from \( x \), \( x_i \) and \( x_j \), as positive samples, and the NT-Xent loss function constructed by \( x_i \) and \( x_j \) will also decrease. In other words, when the model's memory of the sample \( x \) is enhanced, it means that for \( x_i \) and \( x_j \), the model's output tends to be the same, which is determined by the SimCLR framework. \cite{ref15} points out that whether it is a linear or nonlinear projection head, \( f(x_i) \) and \( f(x_j) \) form and maintain more information, and the hidden layers before projection are often richer representations. Therefore, we can explicitly consider that when the sample \( x \) is remembered, the basic encoder's output for \( x_i \) and \( x_j \) will be as similar as possible, so the \( loss() \) can be represented as:
\begin{eqnarray}
	\label{eq12}
	loss\left(x_i,x_j\right)=\Vert f(x_i)-f(x_j)\Vert_2
\end{eqnarray}

The notation \( \lVert \cdot \rVert_2 \) represents the L2 norm. The construction of the loss function \( \text{loss}() \) based on \( f(x_i) \) and \( g(f(x_i)) \) will be discussed in Section~\ref{4.3}.

However, for the first term of Equation~(\ref{eq11}), we found that in datasets with a sufficiently large number of samples, for most samples, this term is almost uniform. As shown in Figure~(\ref{fig4}) and (\ref{fig5}), we conducted two experiments: a) We randomly selected 50 samples \( \{x^{(1)}, \ldots, x^{(i)}, \ldots, x^{(50)}\} \) and examined the $\begin{matrix}loss(x_{i},x_{j})\\A(D-B)\\\end{matrix}$ for each sample. b) To explore the impact of the number of sampled data on the original dataset, we randomly sampled a certain proportion of samples from the entire dataset, denoted as \( B_0 = r \times D_0 \), where \( B \) is the set of samples, and we investigated the relationship between the value of $\mathbb{E}_{x\in B}\begin{matrix}loss(x_{i},x_{j})\\A(D-B)\\\end{matrix}$ and the proportionality factor \( r \).
\begin{figure}[h]
	\centering
	\includegraphics[width=0.6\textwidth]{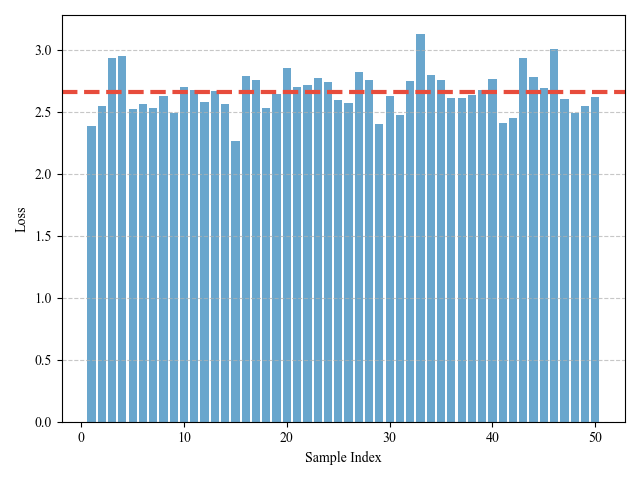}
	\caption{In CIFAR10-ResNet50, the loss statistics of a single sample when randomly removed were conducted over 50 experiments, with the red line representing the mean loss of the 50 samples.
	}
	\label{fig4}
\end{figure}

\begin{figure}[h]
	\centering
	\includegraphics[width=0.6\textwidth]{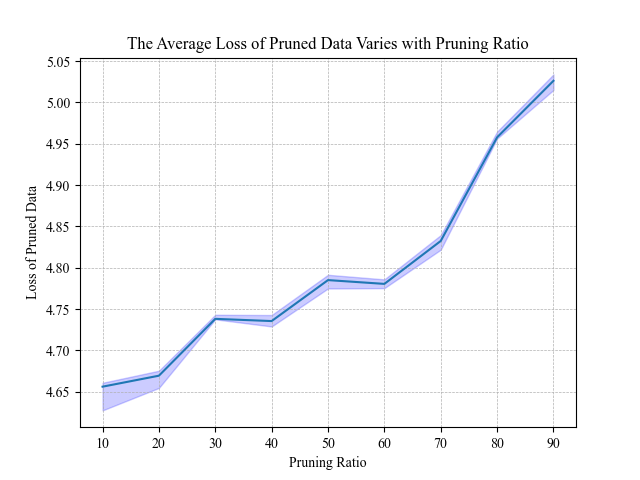}
	\caption{Using the CIFAR10 dataset and the ResNet50 model, static pruning was performed by randomly sampling data. The figure reports the average loss of the pruned data at different pruning rates, with each data point run 5 times, and the shaded area represents the error range.
	}
	\label{fig5}
\end{figure}
According to the experiments, when a single sample is removed, the variation of \( \text{loss}(x_i, x_j) \) with \( A(D - x) \) is not significant. We attribute this to the model having already learned sufficient knowledge from most of the data, and thus is not sensitive to these individual samples. Therefore, we can fully consider that for different data, the variation of $\begin{matrix}loss(x_{i},x_{j})\\A(D-B)\\\end{matrix}$ can be neglected. Moreover, as the sampling ratio \( r \) increases, this value of $\begin{matrix}loss(x_{i},x_{j})\\A(D-B)\\\end{matrix}$ becomes more pronounced, which makes the first term in Equation~(\ref{eq11}) non-negligible. In extreme cases, the performance of the data pruning method will drop sharply, reflecting the difficulty of data pruning in CL at extreme pruning rates, which will be our future work target.

Based on the above analysis, we disregard the first term of Equation~(\ref{eq11}), and drawing inspiration from the addition of a memory term to the scoring function in SL, we construct the scoring function as follows:
\begin{eqnarray}
	\label{eq13}
	score\left(\left(x^{\left(i\right)},y^{\left(i\right)}\right)\right)=NX\left(g\left(f\left(x_i\right)\right),g\left(f\left(x_j\right)\right)\right)-\beta loss\left(x_i,x_j\right)
\end{eqnarray}
Where \( NX(\cdot) \) represents the NT-Xent loss, and \( \beta \) is an adjustable hyperparameter.

The data pruning method in SSL is presented in Algorithm~\ref{alorithm 2}.
\begin{algorithm}[h]
	\caption{EMP in SSL}
	\label{alorithm 2}
	\begin{algorithmic}[1]
		% {\bfseries Output:}  $\mathbf{c}$
		\Require Dataset \( D \), Epoch \( T \), Pruning Ratio \( s \), Initial Model \( \theta \), Retained Dataset \( \hat{D}_0 = D \).
		
		\hrule
		\vspace{0.1cm}
		\For{$t$ from 0 to $T-1$}
		\If{$t!=0$}
		\State Select the retained dataset \( \hat{D}_t \) based on the pruning rate \( s \) and sample scores.
		
		\Else
		\State $\hat{D}_t=D$
		\EndIf
		\State Get train sequence $\hat{D}_t=\{B_0,B_1,,,B_{b_t-1}\}$
		\For{$j$ from 0 to $b_t-1$}
		\State Get NX-loss $NX\left(g\left(f\left(B_j\right)\right),g\left(f\left(B_j\right)\right)\right)$
		\State Update score with Equation~(\ref{eq13})
		\State Update model with optimizer
		\EndFor
		\EndFor
		\vspace{0.1cm}
	\end{algorithmic}
\end{algorithm}

The limitation of static pruning lies in the fact that methods like Yang et al.\cite{yangdataset} rely on proxy models. This introduces architectural bias - the agent's "perception" of the importance of data may not be transferred to the target model, thereby undermining generalization.

EMP continuously scores and trims the data at each training checkpoint (Algorithms \ref{alorithm 1} and \ref{alorithm 2}). By evaluating samples based on the current model state, it retains the data most relevant to the model's instantaneous learning stage. This ensures that the training subset evolves along with the model and maintains consistency with its generalization trajectory.

The dynamic approach of EMP transforms pruning from a one-time filtering to a continuous feedback loop between the model state and data utility. By aligning subset selection with the real-time optimization dynamics of the model and explicitly enhancing memory, emp avoids the distribution mismatch and representation bias that plague static methods. This is crucial for high pruning rates, as static subsets are increasingly inconsistent with the evolving generalization requirements of the model.

\section{Experiment}
In the following sections, we validate the effectiveness of our theoretical results and the proposed dataset pruning method through experiments. In Section~\ref{4.1}, we compare EMP with several other baseline methods on SL, including image classification tasks and a range of natural language tasks. In Section~\ref{4.2}, we verify the high efficiency of EMP on SSL. In Section~\label{ 4.3}, we conduct a series of ablation experiments to validate the effectiveness of our theoretical results. In Section~\ref{4.4}, we analyze the generalization performance of EMP using a one-dimensional linear interpolation method. It is worth noting that other baseline methods include static pruning: CD \cite{ref30}, Herding \cite{ref31}, DeepFool \cite{ref32}, Last Confidence \cite{ref48}, Glister \cite{ref51}, and dynamic pruning: InfoBatch \cite{ref4}, \(\epsilon\)-greedy, and UCB \cite{ref5}. It is important to note that EMP is also a dynamic pruning method, so we mainly compare it with other dynamic pruning methods. To eliminate the influence of randomness, each of our experiments was run 5 times, and the average was taken. Our dataset introduction and experimental details are in the Appendix~\ref{app_2} and \ref{app_3}.

\subsection{Performance of EMP in SL}\label{4.1}
\textbf{CIFAR.} We conducted comparisons on both CIFAR-10 and CIFAR-100. In this work, similar to InfoBatch \cite{ref4}, we employed an annealing technique, specifically, using the entire dataset for training towards the end of the training process. More details are shown in Section~\ref{4.3}. We presented results from other methods in recent years in Table~(\ref{table1}). EMP achieved lossless performance on CIFAR-100, even slightly higher than the baseline. Notably, EMP significantly surpassed previous static pruning methods and led other dynamic pruning methods at high pruning rates. Specifically, on the CIFAR-100 dataset, EMP exceeded other methods by at least 1.3\% under any pruning rate, and even by at least 2.1\% at a 70\% pruning rate.
\setlength{\tabcolsep}{2.5pt}
\begin{table}[h]
	\centering
	\caption{Comparison of EMP with other dataset pruning methods on ResNet-18. Pruning methods are categorized into static pruning and dynamic pruning. "Random*" refers to random dynamic dataset pruning. "Baseline" refers to the performance without pruning.}
	\label{table1}
	\begin{tabularx}{\textwidth}{Xccccccc}  % ?? tabularx ?????
		\toprule
		& Dataset & \multicolumn{3}{c}{CIFAR10} & \multicolumn{3}{c}{CIFAR100} \\
		\cmidrule(lr){1-2} \cmidrule(lr){3-5} \cmidrule(lr){6-8}
		&	Pruning Ratio \% & 30 & 50 & 70 & 30 & 50 & 70 \\
		\cmidrule(lr){1-2} \cmidrule(lr){3-5} \cmidrule(lr){6-8} 
		\multirow{11}{*}{\begin{sideways}Static\end{sideways}}
		&Random & 94.6 & 93.3 & 90.2 & 73.8 & 72.1 & 69.7 \\
		&CD(\cite{ref30}) & 95.0 & 94.3 & 90.8 & 74.2 & 72.3 & 70.3 \\
		&Herding(\cite{ref31})&92.2&88.0&80.1&73.1&71.8&69.6\\
		&K-Center(\cite{ref47})&94.7&93.9&90.9&74.1&72.2&70.2\\
		&Least Confidence(\cite{ref48}) & 95.0 & 94.5 & 90.3 & 74.2 & 72.3 & 69.8 \\
		&Margin(\cite{ref48}) & 94.9 & 94.3 & 90.9 & 74.0 & 72.2 & 70.2 \\
		&GraNd-4(\cite{ref49}) & 95.3 & 94.6 & 91.2 & 74.6 & 71.4 & 68.8 \\
		&DeepFool(\cite{ref32}) & 95.1 & 94.1 & 90.0 & 74.2 & 73.2 & 69.8 \\
		&Craig(\cite{ref50}) & 94.8 & 93.3 & 88.4 & 74.4 & 71.9 & 69.7 \\
		&Glister(\cite{ref51})  & 95.2 & 94.0 & 90.9 & 74.6 & 73.2 & 70.4 \\
		&EL2N-2(\cite{ref3}) & 94.4 & 93.2 & 89.8 & 74.1 & 71.0 & 68.5 \\
		&EL2N-20(\cite{ref3}) & 95.3 & 95.1 & 91.9 & 77.2 & 72.1 & - \\
		&DP(\cite{ref33}) & 94.9 & 93.8 & 90.8 & 77.2 & 73.1 & - \\
		\midrule
		\multirow{5}{*}{\begin{sideways}Dynamic\end{sideways}}
		&Random* & 94.8 & 94.5 & 93.0 & 77.3 & 75.3 & 74.1 \\
		&$\epsilon$-greedy(\cite{ref5})  & 95.2 & 94.9 & 94.1& 76.4 & 74.8 & 75.1 \\
		&UCB(\cite{ref5}) & 95.3 & 94.7 & 93.9 & 77.3 & 75.3 & 74.8 \\
		&InfoBatch(\cite{ref4}) & \textbf{95.6} & 95.1 & 94.7 & 78.2 & 78.1 & 76.5 \\
		&EMP(ours) &95.37\textsubscript{$\pm0.11$}  &\textbf{95.23}\textsubscript{$\pm0.15$}  &\textbf{95.04}\textsubscript{$\pm0.08$}  &\textbf{79.53}\textsubscript{$\pm0.21$}  &\textbf{79.44}\textsubscript{$\pm0.18$} &\textbf{78.63}\textsubscript{$\pm0.42$}  \\
		\midrule
		&Baseline & \multicolumn{3}{c}{95.6$\pm0.1$} & \multicolumn{3}{c}{79.8$\pm0.04$} \\				
		\bottomrule
	\end{tabularx}
\end{table}

Furthermore, Figures~(\ref{fig6}) and (\ref{fig7}) report the variation curves of EMP under different pruning rates. Surprisingly, at low pruning rates, EMP does not perform as well as expected. We believe that at low pruning rates, the model can utilize a larger amount of the dataset through LFL to learn general knowledge and demonstrate strong generalization capabilities. At this time, the model does not have a high demand for repeated learning of key samples. However, as the pruning rate increases and the number of times each sample is trained becomes limited, LFL shows limited performance. In contrast, EMP can extract key samples for the model to learn, strengthen the model's memory of general knowledge, and thus achieve efficient performance.
\begin{figure}[tb]
	\centering
	\includegraphics[width=\textwidth]{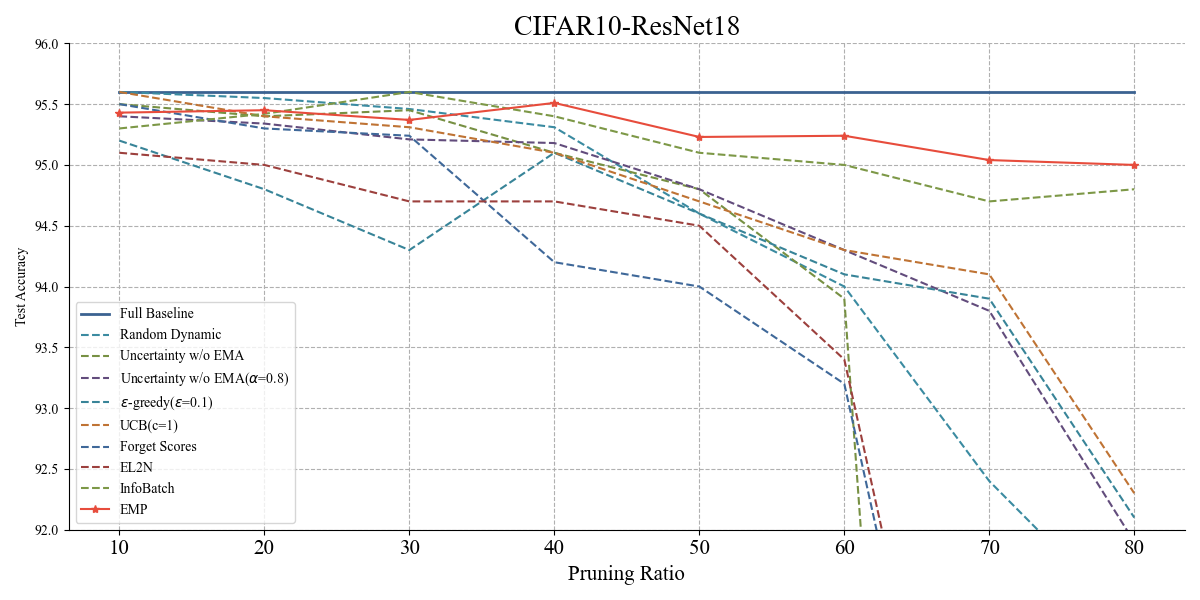}
	\caption{On CIFAR10-ResNet18, a comparison of different pruning methods at various pruning rates is conducted.
	}
	\label{fig6}
\end{figure}
\begin{figure}[tb]
	\centering
	\includegraphics[width=\textwidth]{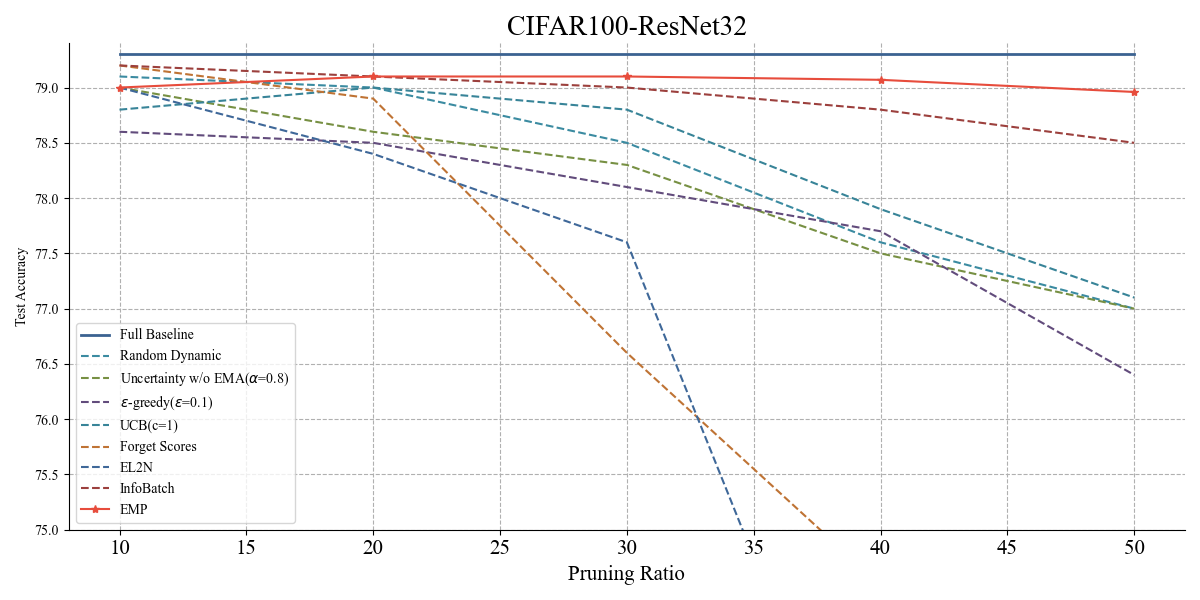}
	\caption{On CIFAR100-ResNet32, a comparison of different pruning rates and various pruning methods is made.
	}
	\label{fig7}
\end{figure}

\textbf{ImageNet-1K.} To explore the performance of EMP on a large-scale dataset, we trained ResNet50 on ImageNet-1K, with results as shown in Table~\ref{table2}. Notably, the UCB and Greedy methods demonstrated performance similar to that of dynamic random pruning, and at high pruning rates, dynamic random pruning surpassed them. We believe that on ImageNet-1K, due to some noisy samples \cite{ref34}, they are difficult to fit correctly, hence LFL tends to select these noisy samples, leading to suboptimal performance of LFL. As expected, on the challenging ImageNet-1K dataset with high fitting difficulty, EMP showed efficient performance both at low and high pruning rates.

\begin{table}[tb]
	\centering
	\caption{EMP was evaluated alongside other dynamic pruning methods on ImageNet-1k using ResNet50. "Baseline" refers to the performance without pruning. Random represents static random pruning, and Random* represents dynamic random pruning.}
	\label{table2}
	\begin{tabular}{ccc}
		\toprule
		\multicolumn{3}{c}{ImageNet-1k-ResNet50}\\
		\midrule
		Pruning Ratio \% & 30 & 70 \\
		\midrule
		Random&72.15&68.52\\
		Random*&73.06&70.54\\
		UCB&72.97&68.59\\
		$\epsilon$-greedy&74.10&69.28\\
		InfoBatch&74.59&72.51\\
		
		EMP(ours)&\textbf{74.79}\textsubscript{$\pm$0.15}&\textbf{72.77}\textsubscript{$\pm$0.11}\\
		\midrule
		Baseline&\multicolumn{2}{c}{75.36\textsubscript{$\pm$0.15}}\\
		\bottomrule
	\end{tabular}
\end{table}

\textbf{GLUE.} In addition to image classification tasks, natural language understanding tasks are also within the scope of our investigation. Specifically, we used the BERT-base pre-trained model \cite{ref24} for fine-tuning on the GLUE benchmark \cite{ref35}. The experimental results are reported in Table~(\ref{table3}), and it is evident that EMP leads in most cases, even outperforming the case without pruning. Notably, at pruning rates of 50\% and 70\%, for datasets with a larger number of samples such as MNLI (393K) and QQP (363K), EMP demonstrates good performance. Specifically, at a 70\% pruning rate, EMP leads other methods by at least 1\% on the MNLI dataset. On the dataset QNLI with a moderate number of samples (108K), EMP also shows superior performance, exceeding other methods by 0.6\%-1.9\% at a 70\% pruning rate. This proves that our method is also effective on natural language tasks.
\setlength{\tabcolsep}{1.6pt}
\begin{table}[h]
	%\begin{center}
		\caption{Comparison of Dynamic Pruning Methods on GLUE}
		\label{table3}
		\begin{tabularx}{\textwidth}{clcccccccc}
			\toprule
			&Dataset & RTE & CoLA & SST-2 & STS-B & MRPC & QNLI & MNLI & QQP \\
			&Eval Metric&Acc&Matthew's Cor.&Acc&Pearson Cor.&Acc&Matched Acc.&Acc&Acc\\
			\midrule
			&Whole dataset & 67.08 & 57.61 & 92.78 & 88.76 & 86.24 & 89.15 & 84.37 & 91.10 \\
			\midrule
			\multirow{4}{*}{\begin{tabular}[c]{@{}c@{}}30\%\end{tabular}}
			&InfoBatch & \textbf{67.06} & \textbf{59.08} & 92.89 & 88.26 & 84.38 & \textbf{91.26} & \textbf{84.40} & 91.32 \\
			&$\epsilon$-greedy & 64.78 & 58.55 & \textbf{93.30} & 88.59 & 84.94 & 90.89 & 84.31 & 91.44 \\
			&UCB & 64.33 & 57.41 & 93.08 & 88.63 & 84.78 & 90.14 & 84.25 & 91.54 \\
			&EMP(ours) & 66.33 & 58.96 & 93.18 & \textbf{88.79} & \textbf{85.56} & 91.24 & 84.27 & \textbf{91.61} \\
			\midrule
			\multirow{4}{*}{\begin{tabular}[c]{@{}c@{}}50\%\end{tabular}}
			&InfoBatch & 68.21 & 58.12 & 92.86 & 88.85 & 84.94 & 90.55 & 84.24 & 91.35 \\
			&$\epsilon$-greedy & 68.17 & \textbf{60.39} & 93.31 & 88.26 & 85.33 & 90.11 & 84.14 & 91.25 \\
			&UCB & \textbf{68.88} & 57.61 & \textbf{93.98} & 88.10 & 82.61 & 88.64 & 84.17 & 90.69 \\
			&EMP(ours) & 68.43 & 59.17 & 93.74 & \textbf{89.11} & \textbf{85.61} & \textbf{91.22} & \textbf{84.25} & \textbf{91.53} \\
			\midrule
			\multirow{4}{*}{\begin{tabular}[c]{@{}c@{}}70\%\end{tabular}}
			&InfoBatch & 64.67 & 58.21 & 92.52 & 87.38 & 81.65 & 89.12 & 82.18 & 91.02 \\
			&$\epsilon$-greedy & 64.94 & 56.59 & 92.63 &\textbf{88.49} & 85.09 & 89.54 & 82.53 & 90.68 \\
			&UCB & 63.66 & 55.69 & 92.86 & 87.25 & 45.83 & 88.23 & 79.94 & 89.21 \\
			&EMP(ours) & \textbf{65.67} & \textbf{58.47} & \textbf{93.11} & 88.41 & \textbf{85.75} & \textbf{90.19} & \textbf{83.10} & \textbf{91.33} \\
			\bottomrule
		\end{tabularx}
	%\end{center}
\end{table}

\subsection{Performance of EMP in SSL}\label{4.2}
Recently, Large Language Models (LLMs) have demonstrated significant performance \cite{ref25,ref23}, pretraining on vast datasets to obtain models with general knowledge. Visually, various visual pretrained models \cite{ref36,ref37} are also gaining popularity. To explore the performance of EMP in model pretraining, we experimentally verified it on CL using the SimCLR framework \cite{ref15}. In this work, we specifically pruned the data during the pretraining phase to investigate the practical effectiveness of EMP.

We conducted explorations on the CIFAR10/100 datasets, with results shown in Table~(\ref{table4}). Specifically, for the CIFAR10 dataset, EMP outperformed other methods at different pruning ratios, especially at pruning ratios of 50\% and 70\%, where EMP achieved accuracies of 82.16\% and 80.59\%, respectively, significantly higher than other methods. In the CIFAR100 dataset, EMP also demonstrated superior performance across all pruning ratios, with accuracies of 56.37\%, 51.99\%, and 48.21\% at pruning ratios of 30\%, 50\%, and 70\%, respectively, showing a notable improvement over other methods such as InfoBatch.

\begin{table}[h]
	\centering
	\caption{Using the SimCLR framework on the CIFAR10/100 dataset with the ResNet50 model, different data pruning methods are applied during the model's pre-training phase, and the performance after fine-tuning is presented. "Baseline" refers to the performance without pruning.}
	\label{table4}
	%\resizebox{\columnwidth}{!}{
		\begin{tabular}{cccccccc}
			\toprule
			& Dataset & \multicolumn{3}{c}{CIFAR10} & \multicolumn{3}{c}{CIFAR100} \\
			\cmidrule(lr){1-2} \cmidrule(lr){3-5} \cmidrule(lr){6-8}
			&	Pruning Ratio \% & 30 & 50 & 70 & 30 & 50 & 70 \\
			\cmidrule(lr){1-2} \cmidrule(lr){3-5} \cmidrule(lr){6-8}

			&$\epsilon$-greedy(\cite{ref5})  & 82.01 &79.98 &76.91 & 53.11 & 49.89 & 42.72 \\
			&UCB(\cite{ref5}) &81.37 & 79.24 & 71.85 & 52.50 & 49.34 &42.15 \\
			&InfoBatch(\cite{ref5}) & 83.09 & 81.39 & 79.05 & 55.21 & 50.56 & 46.01 \\
			&EMP(ours) &\textbf{83.19}\textsubscript{$\pm0.33$}  &\textbf{82.16}\textsubscript{$\pm0.37$}  &\textbf{80.59}\textsubscript{$\pm0.28$}  &\textbf{56.37}\textsubscript{$\pm0.41$}  &\textbf{51.99}\textsubscript{$\pm0.52$} &\textbf{48.21}\textsubscript{$\pm0.47$}  \\
			\midrule
			&Baseline & \multicolumn{3}{c}{84.67\textsubscript{$\pm0.17$}} & \multicolumn{3}{c}{56.92\textsubscript{$\pm0.44$}} \\				
			\bottomrule
		\end{tabular}
	%}
\end{table}

Overall, in CL, the EMP method is capable of maintaining high model performance while reducing the amount of training data, especially outstanding at high pruning ratios. This indicates that the EMP method has strong robustness and effectiveness in dynamic dataset pruning.

\subsection{Ablation Experiment}\label{4.3}
\textbf{Annealing Technique.} In SSL, for the CIFAR10/100 datasets, we employed an annealing technique. Specifically, we determined a hyperparameter \( \alpha \) to control the degree of annealing, and \( \alpha \) is defined as follows:
\begin{eqnarray}
	\alpha=\frac{annealing\ epochs}{total\ epochs}
\end{eqnarray}

It can be seen that the degree of annealing increases with the increase of \( \alpha \), but this also brings corresponding additional overhead. We believe that when the model has acquired sufficient knowledge in the non-annealing phase, the positive effects brought by the annealing phase will be less significant. In other words, the annealing phase not only exists as a technique to improve accuracy but also serves as an effective measure of the validity of the data pruning algorithm. Similarly, InfoBatch \cite{ref4} also uses this technique, and we mainly compare it with this method. Our results are reported in Figure~(\ref{fig8}). It is not difficult to see that at lower pruning rates, both methods are robust to annealing, and when the pruning rate increases, the shortcomings of LFL become apparent. Although EMP's accuracy drops without the annealing phase, InfoBatch's drop is more dramatic. Therefore, we can conclude that under any circumstances, EMP is robust to the annealing technique, which means that EMP enables the model to remember more samples and gain more knowledge during the early to mid-training phase.

\begin{figure}[h]
	\centering
	\includegraphics[width=0.7\textwidth]{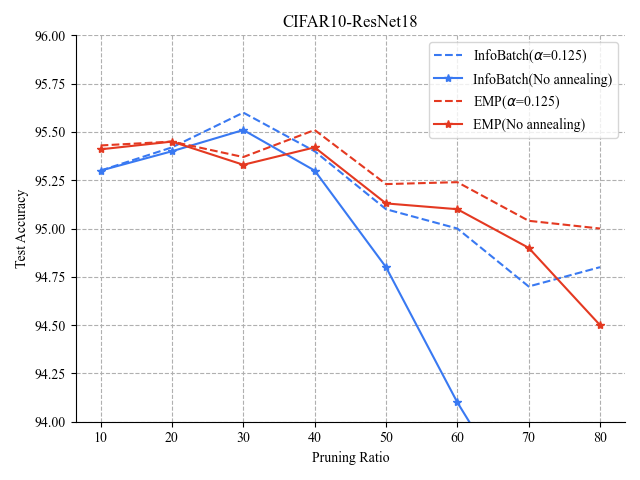}
	\caption{A comparison of accuracy with and without the use of annealing technology at different pruning rates is presented, in which the InfoBatch method also employs annealing technology.
	}
	\label{fig8}
\end{figure}

\textbf{Memory Term Construction.} In Section~\ref{3.4}, we posit that the hidden layers before projection often contain richer representations, hence utilizing the pre-projection \(f(x_i)\) and \(f(x_j)\) to construct the memory term. In this work, we construct memory terms for both the hidden layer outputs before and after projection. Specifically, we compare the performance of two types of memory terms, \(loss_f\) and \(loss_g\), which are represented as follows:
\begin{eqnarray}
	loss_f=\Vert f(x_i)-f(x_j)\Vert_2
\end{eqnarray}
\begin{eqnarray}
	loss_g=\Vert g(f(x_i))-g(f(x_j))\Vert_2
\end{eqnarray}

This corresponds to Equation~(\ref{eq11}) in Section~\ref{3.4}. Our results are reported in Table~(\ref{table5}). It can be observed that using \(loss_g\) as the memory term results in a decrease in performance, which confirms our previous idea. In other words, since \(f(x_i)\) and \(f(x_j)\) form and maintain more information, constructing the memory term with the more informative \(f(x_i)\) and \(f(x_j)\) is more effective.

\textbf{Label Noise.} Label noise is a common challenge in real-world applications. Enhancing the robustness of algorithms to label noise during the data pruning process is a key issue \cite{ref17}. In this section, we investigate the robustness of EMP to label noise by conducting comparative experiments on CIFAR10 and CIFAR100 with synthetic label noise. Specifically, we introduce label noise \cite{ref38} into the two datasets by randomly replacing the labels of a certain percentage of training data with all possible labels, which mimics the real-life scenario where researchers may misjudge images \cite{ref39}. Our results are reported in Figure~(\ref{fig9}). It can be observed that EMP still maintains high performance even with the addition of 20\% noise and leads other methods. Notably, the Greedy and UCB methods are not robust to noise, and are even outperformed by random dynamic pruning. We believe that these noisy data are difficult to fit during training, hence their scores remain high, leading to the retention of noisy data and resulting in low model performance. Overall, EMP maintains high performance in the presence of noise, indicating that EMP indeed retains key or general samples, thus demonstrating robustness to noise.

\begin{table}[h]
	\centering
	\caption{In contrastive learning, a performance comparison of constructing memory terms for different hidden layers.}
	\label{table5}
	%\resizebox{\columnwidth}{!}{
		\begin{tabular}{cccccccc}
			\toprule
			& Dataset & \multicolumn{3}{c}{CIFAR10} & \multicolumn{3}{c}{CIFAR100} \\
			\cmidrule(lr){1-2} \cmidrule(lr){3-5} \cmidrule(lr){6-8}
			&	Pruning Ratio \% & 30 & 50 & 70 & 30 & 50 & 70 \\
			\cmidrule(lr){1-2} \cmidrule(lr){3-5} \cmidrule(lr){6-8}

			&$loss_g$ & 83.11 & 81.52 & 79.22 & 54.85 & 50.40 & 47.58 \\
			&$loss_f$(ours) &\textbf{83.19}\textsubscript{$\pm0.33$}  &\textbf{82.16}\textsubscript{$\pm0.37$}  &\textbf{80.59}\textsubscript{$\pm0.28$}  &\textbf{56.37}\textsubscript{$\pm0.41$}  &\textbf{51.99}\textsubscript{$\pm0.52$} &\textbf{48.21}\textsubscript{$\pm0.47$}  \\
			\midrule
			&Baseline & \multicolumn{3}{c}{84.67\textsubscript{$\pm0.17$}} & \multicolumn{3}{c}{56.92\textsubscript{$\pm0.44$}} \\				
			\bottomrule
		\end{tabular}
	%}
\end{table}

\begin{figure}[tb]
	\centering
	\includegraphics[width=\textwidth]{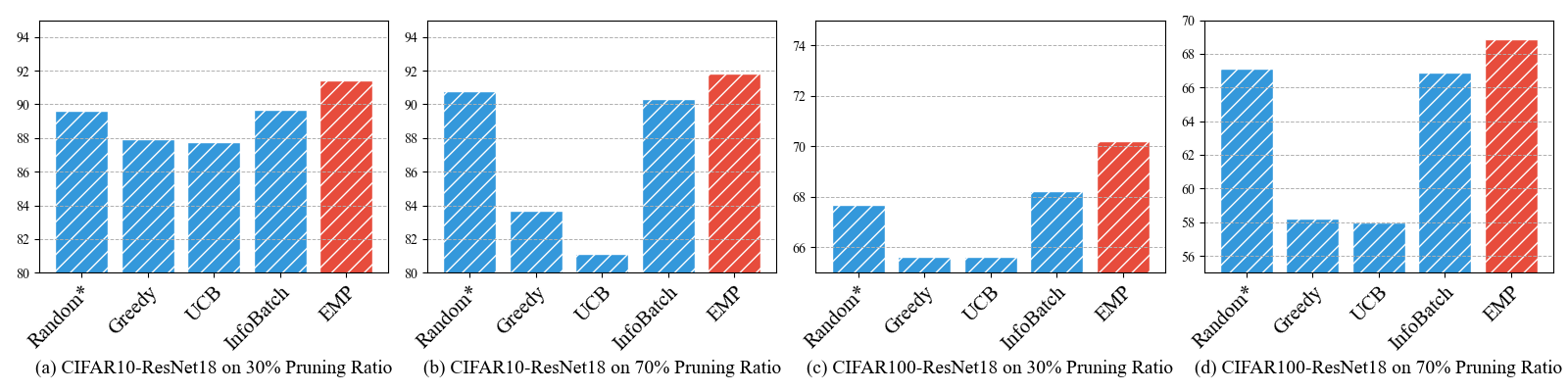}
	\caption{Performance of different methods under the condition of injecting 20\% random label noise into the samples, where "Random*" indicates random dynamic pruning.
	}
	\label{fig9}
\end{figure}

\textbf{Hyperparameter \( \beta \).} In our method, both in SL and SSL, we have added a memory term to the scoring function, which is scaled by the hyperparameter \( \beta \). Our experiments indicate that the best results are achieved when \( \beta \) is set to 5. To explore the impact of \( \beta \) on the experimental results, we conducted experiments with different values of \( \beta \), and the results are reported in Figure~(\ref{fig10}). It is evident that even with different values of \( \beta \), EMP still leads in most cases, demonstrating that the memory term we added can be effective. When \( \beta \) exceeds the value we set, performance declines, which can be interpreted as the model having sufficient memory for these key samples but lacking in the overall information from the dataset, leading to a decrease in performance.
\begin{figure}[tb]
	\centering
	\includegraphics[width=\textwidth]{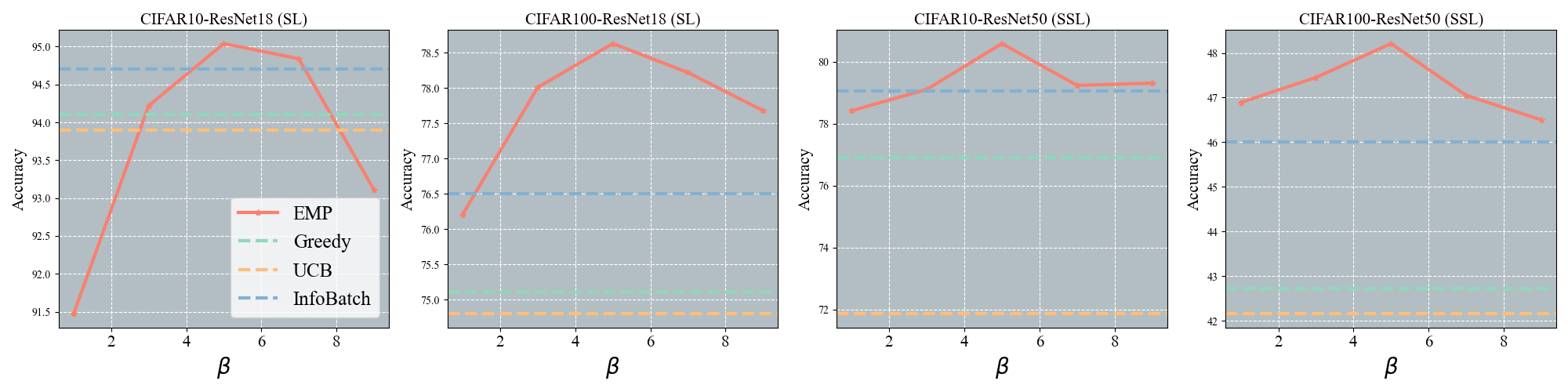}
	\caption{At a 70\% pruning rate, the effect of varying \( \beta \) values on the performance of different tasks is assessed. SL denotes Supervised Learning, and SSL denotes Self-Supervised Learning, where we employ the SimCLR framework in the context of SSL.
	}
	\label{fig10}
\end{figure}

\subsection{Generalization Analysis}\label{4.4}
To further investigate the effectiveness of EMP, we analyze its generalization capability. We use one-dimensional linear interpolation to examine how EMP affects the loss landscape while enhancing the model's memory ability. Previous work \cite{ref41,ref42} suggests that better model generalization is associated with flat minima in the loss landscape. Following the method proposed in \cite{ref40}, we inspect the loss landscape. Specifically, we assess the performance of models with parameters \((1-\epsilon)\theta_{\text{init}} + \epsilon\theta_T\), where \(\theta_{\text{init}}\) is the initial model, and \(\theta_T\) is the converged model after optimization with different dataset pruning algorithms.

Our results are reported in Figure~(\ref{fig11}). The performance of EMP is unexpected; in the loss landscape, EMP not only consistently maintains a lower loss value range but also exhibits the flattest performance when \( \epsilon \) is around 1. Moreover, in the accuracy curve, EMP remains leading, and the accuracy only drops rapidly when \( \epsilon \) approaches 1.1. In contrast to EMP, when we retain the samples with low scores in EMP (EMP-reverse), both the loss landscape and the accuracy curve present an opposite landscape to that of EMP. We have every reason to believe that EMP indeed enhances the model's generalization capability.

\begin{figure}[tb]
	\centering
	\includegraphics[width=\textwidth]{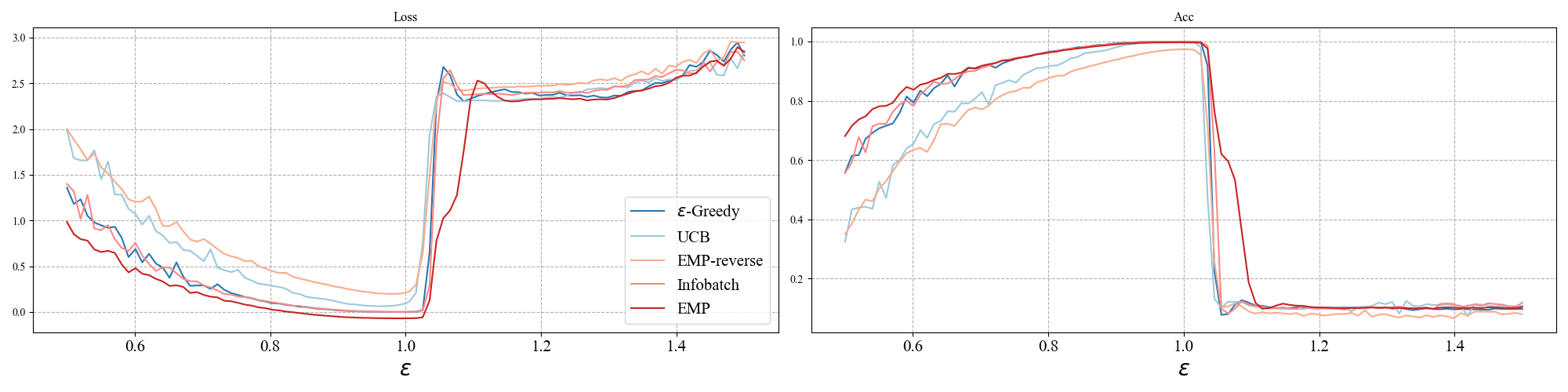}
	\caption{On CIFAR10-ResNet18, with a pruning rate of 50\%, linear interpolation is used to examine the model's loss landscape and accuracy curve.
	}
	\label{fig11}
\end{figure}

\subsection{Class Imbalance}\label{4.5}
To investigate how EMP addresses class imbalance, we conducted an experimental study. The ten classes of CIFAR-10, namely
'airplane', 'automobile', 'bird', 'cat', 'deer;, 'dog', 'frog',
'horse', 'ship', and 'truck', are denoted as C1 to C10 respectively in the following sections of the work. As shown in Figure~ \ref{fig12} and \ref{fig13}, the data subset selected by EMP exhibits significant class imbalance?one class accounts for nearly 20\% of the samples in CIFAR-10. Vysogorets et al.\cite{vysogoretsdrop} point out that the key to effective pruning lies in an appropriate difficulty-based class ratio rather than strict class balance. 

Therefore, we designed a new experiment: using EMP?s sample scoring function while enforcing class balance. The corresponding performance results are reported in Table~\ref{tab1}. It can be observed that the performance of the EMP method significantly deteriorates after class balancing, approaching results comparable to random pruning. This outcome aligns with findings from other methods, indicating that EMP does not enhance performance by resolving class imbalance but rather by identifying samples that genuinely benefit model performance and generalization.

In addition, we designed a new experiment by redistributing class proportions according to the ratios in Table~\ref{tab2} and trained models under this configuration. The results are shown in Figure~ \ref{fig14}. We observe that under high pruning rates, all methods (including EMP) fail to maintain accuracy, as certain sample features cannot be effectively learned. However, we find that EMP outperforms other methods. We argue that in LFL methods (represented by InfoBatch), the model?s higher uncertainty when predicting minority classes leads to more frequent selection of these classes under high pruning rates. Yet, due to missing critical features, these samples often become hard-to-fit. We report this phenomenon in Figure~\ref{fig15}. In contrast, EMP mitigates this limitation via its memory term by selecting more representative samples for training. Through this mechanism, EMP effectively alleviates class imbalance issues.
\begin{figure}[htbp]
	\centering
	\includegraphics[width=0.7\textwidth]{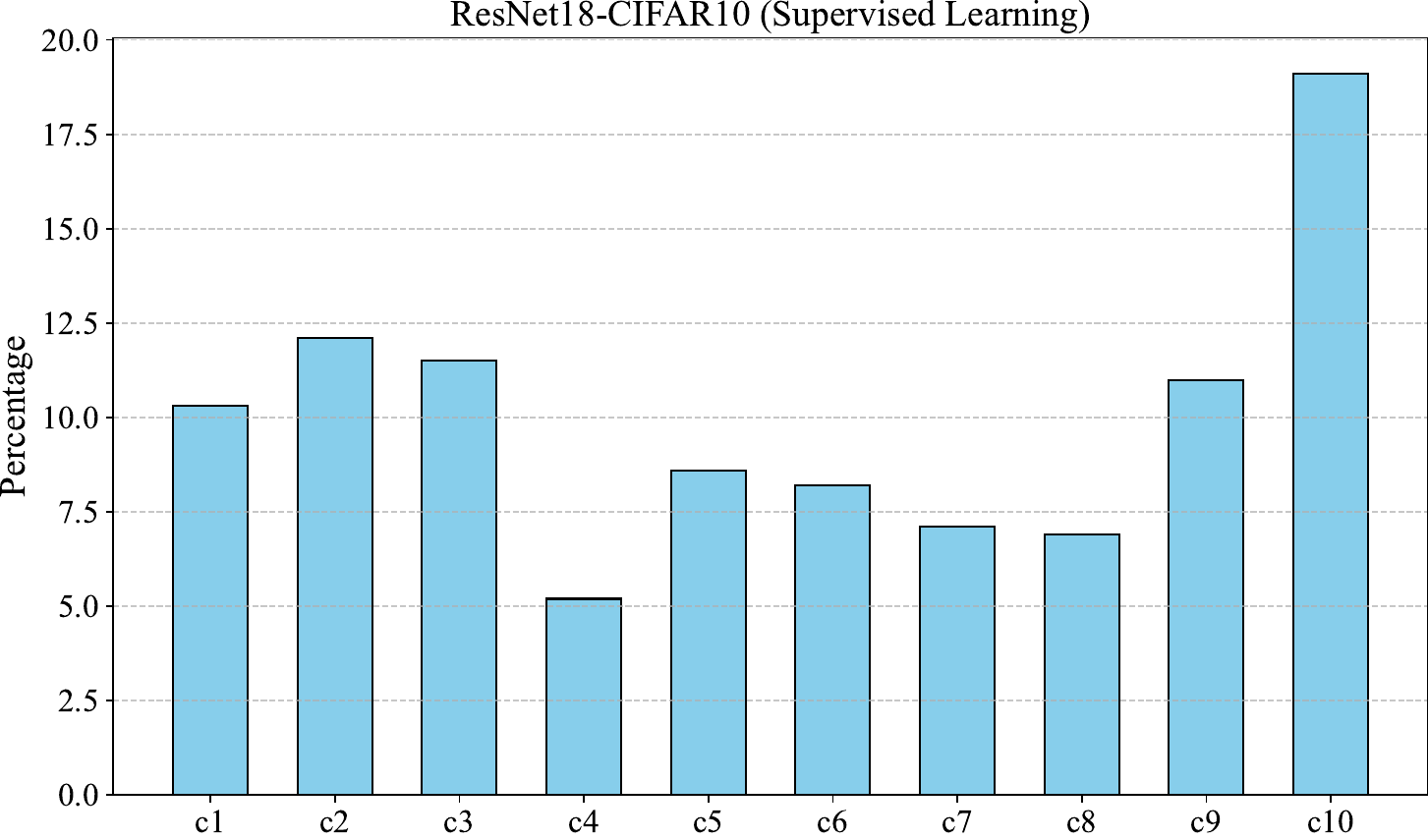}
	\caption{Under supervised learning, we evaluate the percentage of retained samples per class for a ResNet-18 model on the CIFAR-10 dataset at a pruning ratio of 70\%
	}
	\label{fig12}
\end{figure}
\begin{figure}[htbp]
	\centering
	\includegraphics[width=0.7\textwidth]{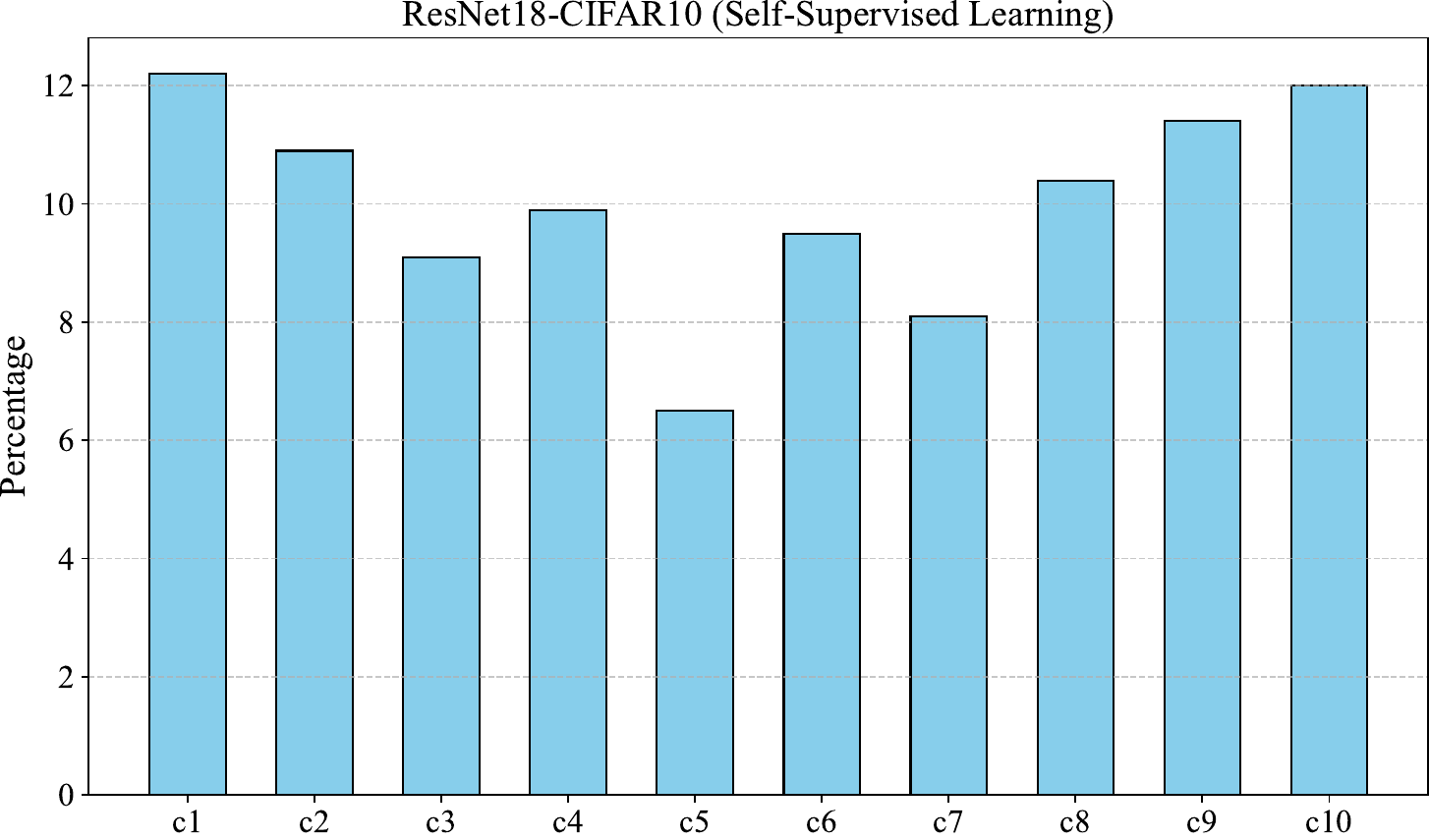}
	\caption{Under self-supervised learning, we evaluate the percentage of retained samples per class for a ResNet-18 model on the CIFAR-10 dataset at a pruning ratio of 70\%
	}
	\label{fig13}
\end{figure}
\begin{figure}[htbp]
	\centering
	\includegraphics[width=0.7\textwidth]{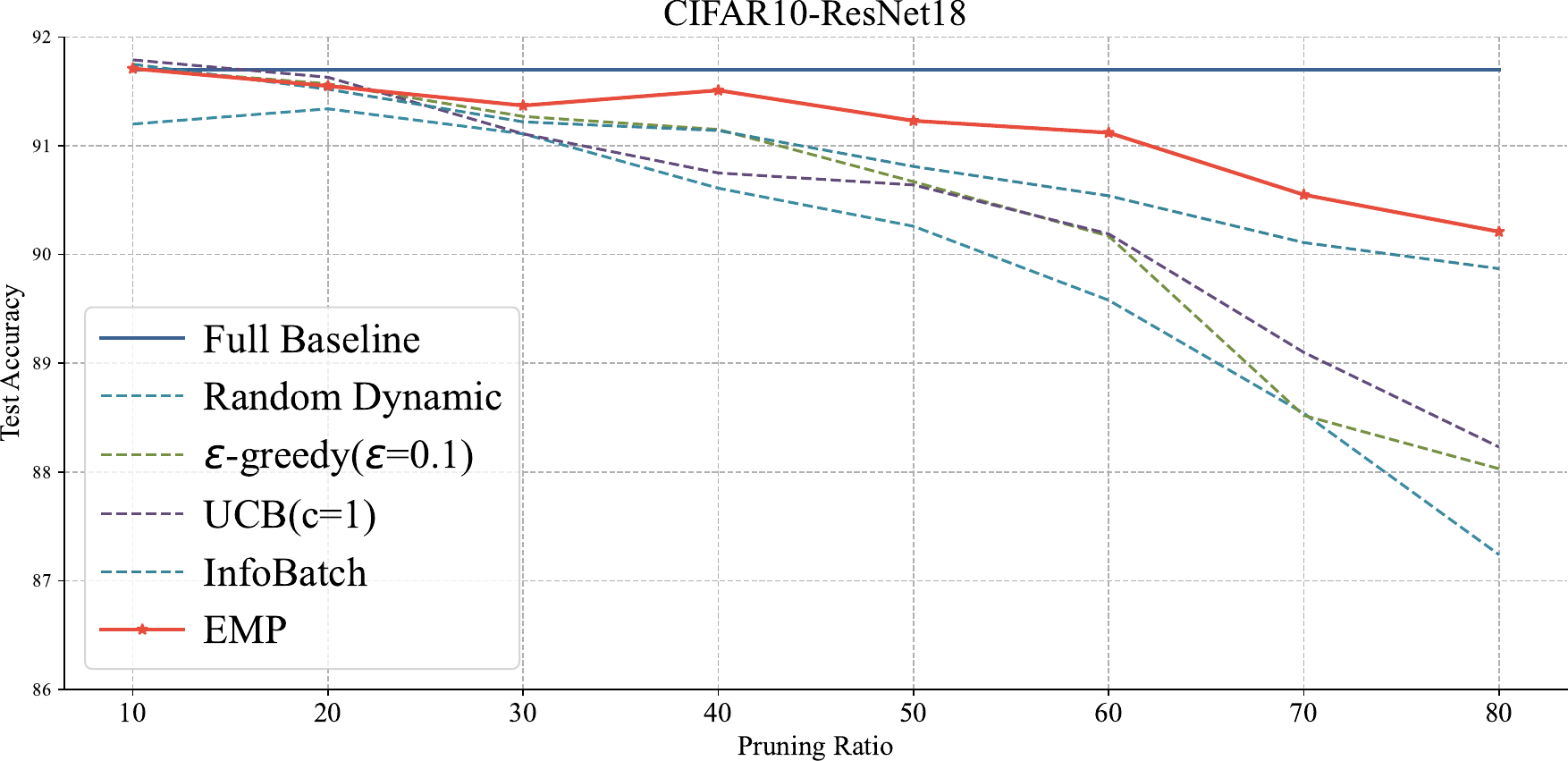}
	\caption{Comparative results of different methods under varying pruning rates using the ResNet10 model and a class-imbalanced CIFAR10 dataset with the class proportions specified in Tab. \ref{tab2}
	}
	\label{fig14}
\end{figure}
\begin{figure}[htbp]
	\centering
	\includegraphics[width=0.7\textwidth]{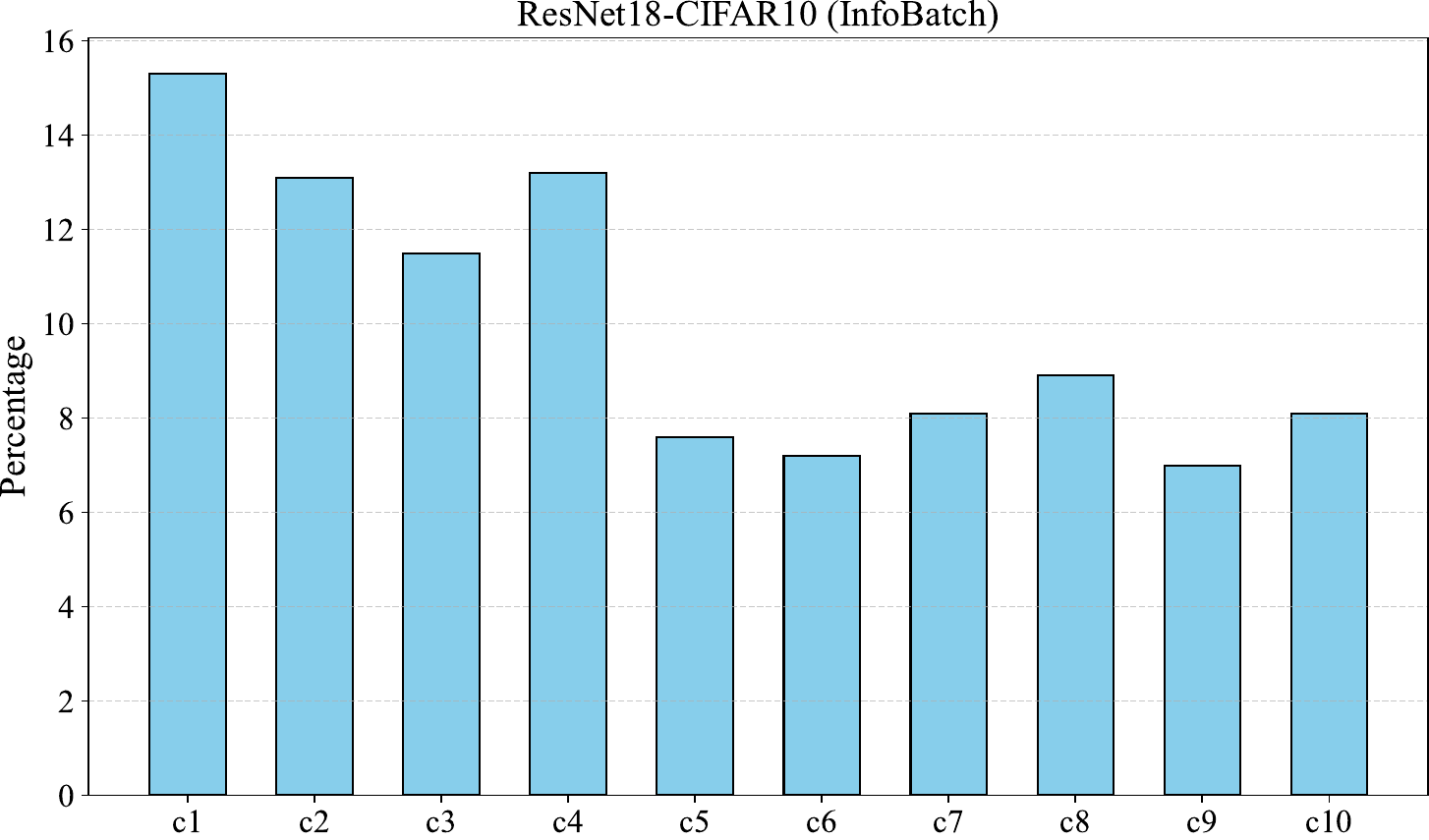}
	\caption{In the InfoBatch method, we evaluated the per-class retention percentage of samples for the ResNet-18 model on the class-imbalanced CIFAR-10 dataset at a pruning rate of 70\%
	}
	\label{fig15}
\end{figure}
\begin{table}[h]
	\centering
	\small
	\caption{Comparative results of EMP\textsubscript{balance} (with enforced equal class proportions) versus other methods on ResNet18 at a 50\% pruning rate}
	\label{tab1}
	\setlength{\tabcolsep}{3mm}{
		\begin{tabular}{c<{\centering}c<{\centering}c<{\centering}c<{\centering}c<{\centering}c<{\centering}c<{\centering}}
			\toprule[1pt]
			&\textbf{EMP\textsubscript{balance}}
			& \textbf{EMP} & \textbf{DP}& \textbf{InfoBatch}&\textbf{Random\textsubscript{dynamic}}\\
			\midrule[0.5pt]
			CIFAR10&94.88&95.23&93.8&95.1&94.5\\
			CIFAR100&72.21&79.44&73.1&78.1&72.1\\
			\bottomrule[1pt]
	\end{tabular}}
\end{table}
\begin{table}[h]
	\centering
	\small
	\caption{The subsampling rates by class for the imbalanced CIFAR10 dataset}
	\label{tab2}
	\setlength{\tabcolsep}{3mm}{
		\begin{tabular}{c<{\centering}c<{\centering}}
			\toprule[1pt]
			\textbf{Class Indices}
			& \textbf{Subsample Rate} \\
			\midrule[0.5pt]
			c1, c2&25\%\\
			c3, c4, c5&50\%\\
			c6, c7& 75\%\\
			c8, c9, c10& 100\%\\
			\bottomrule[1pt]
	\end{tabular}}
\end{table}

\section{Conclusion}
In this work, we proposed the EMP method. We started with the model's memory and theoretically explained the issue of insufficient memory in low-frequency learning. To address this, we identified the memory terms in the model's pre-training and fine-tuning based on theoretical and experimental analysis. EMP solves the problem of insufficient model memory at high pruning rates by adding a memory term to the scoring function. According to experiments, EMP leads the state-of-the-art methods in image classification tasks, natural language understanding tasks, and model pre-training tasks, showing a significant advantage at high pruning rates.

\textbf{Limitations and Future Works.} Previous work \cite{ref45,ref46} has suggested that different layers within a network exhibit varying training dynamics, which also relates to model memory. Shallow layers tend to learn general knowledge, while deeper layers tend to learn task-specific knowledge, with memory primarily occurring in deeper layers \cite{ref9}. We have not yet explored this aspect. Designing or exploring the memory mechanisms and memory terms for structurally different models or different layers within a model is one of our future research directions.
\bmhead{Acknowledgements}
None

\section*{Declarations}

\begin{itemize}
	\item \textbf{Funding. }
	None
	\item  \textbf{Conflict of interest/Competing interests (check journal-specific guidelines for which heading to use). }
	The authors reported no potential conflict of interest.
	\item \textbf{Ethics approval and consent to participate. }Not applicable. 
	
	\item \textbf{Consent for publication. }All authors consent to the publication of this work. 
	
	\item \textbf{Data availability. }The data used in this work are all open-source datasets
	\item \textbf{Materials availability }Not applicable. 
	
	\item \textbf{Code availability. }The code was shown in https://github.com/xiaojinying/EMP
	\item \textbf{Author contribution. }JinYing Xiao conceived the research idea and developed the study design. Ping Li wrote the main manuscript text, developed the methodology, and contributed to software implementation. Jie Nie validated the results and contributed to software development. Zhe Tang prepared the visualizations. All authors reviewed and edited the manuscript.
\end{itemize}

\begin{appendices}
\renewcommand{\theequation}{A.\arabic{equation}}
\section{Proof}\label{app_1}
\setcounter{equation}{0}

In Section~\ref{3.2}, we decomposed \( I(\theta; Y | X) \) and derived the following:
\begin{eqnarray}
	I\left(\theta;Y\middle| X\right)=H\left(\theta\right)+H\left(Y\middle| X\right)-H\left(\theta,Y\middle| X\right)
\end{eqnarray}
Here, we demonstrate that \( H(\theta, Y | X) \leq H(\theta) \), thereby establishing a lower bound for \( I(\theta; Y | X) \).
The proof process is as follows.

The Data Processing Inequality states that for any random variables \(X, Y, Z\), the following is true: 
\begin{eqnarray}
	H\left(Y\mid X\right)\geq H\left(Y\mid X,Z\right)
\end{eqnarray}

Substituting \(Z = \theta\) into the above formula, we get: 
\begin{eqnarray}
	\label{app_eq2}
	H\left(Y\mid X\right)\geq H\left(Y\mid X,\theta\right)
\end{eqnarray}

The chain rule for joint entropy can be written as: 
\begin{eqnarray}
	\label{app_eq3}
	H\left(\theta,Y\mid X\right)=H\left(\theta\mid X\right)+H\left(Y\mid\theta,X\right)
\end{eqnarray}

Combining Equation~(\ref{app_eq2}) with Equation~(\ref{app_eq3}), we can obtain:
\begin{eqnarray}
	H\left(Y\mid\theta,X\right)\le H\left(Y\mid X\right)
\end{eqnarray}

According to the definition of joint entropy, expanding \( H(\theta, Y | X) \) yields:
\begin{eqnarray}
	H\left(\theta,Y\mid X\right)=H\left(\theta\mid X\right)+H\left(Y\mid\theta,X\right)\le H\left(\theta\mid X\right)+H\left(Y\mid X\right)
\end{eqnarray}

Since entropy is non-negative, therefore:
\begin{eqnarray}
	\label{app_eq6}
	H\left(\theta,Y\mid X\right)\le H\left(\theta\mid X\right)
\end{eqnarray}

Since \( H(\theta | X) \) is the uncertainty of \( \theta \) given \( X \), it cannot be greater than the overall uncertainty of \( \theta \), \( H(\theta) \), thus we obtain:
\begin{eqnarray}
	\label{app_eq7}
	H\left(\theta\mid X\right)\le H\left(\theta\right)
\end{eqnarray}

Combining Equation~(\ref{app_eq6}) with Equation~(\ref{app_eq7}), we conclude:
\begin{eqnarray}
	H\left(\theta,Y\mid X\right)\le H\left(\theta\right)
\end{eqnarray}

\section{Dataset Introduction}\label{app_2}
\renewcommand{\thetable}{B.\arabic{table}}
\setcounter{table}{0} 
\subsection{Image Classification Dataset}
\textbf{CIFAR10.} The CIFAR-10 dataset \cite{ref54} consists of 60,000 32x32 color images, divided into 50,000 training images and 10,000 test images. The dataset is divided into 10 classes, including: airplanes, cars, cats, etc. Each class contains 6,000 images.

\textbf{CIFAR100.} The CIFAR-100 dataset is similar in scale and image size to CIFAR-10 but contains 100 classes, with 600 images per class, totaling 60,000 images. These 100 classes are organized into 20 superclasses. The categories in CIFAR-100 are more granular; for example, what is categorized as "cats" in CIFAR-10 falls under the pet superclass in CIFAR-100.

\textbf{ImageNet-1K.} The ImageNet dataset was initially created to support the ImageNet Challenge (ILSVRC) and is one of the largest image databases to date, containing over 14 million annotated images spanning more than 20,000 categories.

\subsection{Natural Language Datasets}
\textbf{GLUE.} The General Language Understanding Evaluation (GLUE) benchmark \cite{ref35} consists of nine natural language understanding (NLU) tasks, all in English. The GLUE tasks include single-sentence tasks like CoLA and SST-2, similarity and paraphrase tasks such as MRPC, STS-B, and QQP, as well as natural language inference tasks including MNLI, QNLI, RTE, and WNLI. Well-known models like BERT \cite{ref24}, XLNet \cite{ref52}, RoBERTa \cite{ref53}, ERINE, T5, and others are tested on this benchmark. Specific details of each dataset are reported in Table~(\ref{app_table1}).

\begin{table}[h]
	\centering
	\caption{GLUE description and statistics}
	\label{app_table1}
	\begin{tabular}{lrrlll}
		\toprule
		Corpus & |Train| & |Test| & Task & Metrics & Domain \\
		\midrule
		\multicolumn{6}{c}{Single-Sentence Tasks} \\ 
		\midrule
		CoLA & 8.5k & 1k & acceptability &  Matthews corr. & misc. \\
		
		SST-2 & 67k & 67k & sentiment &  acc. & movie reviews \\
		
		\midrule
		\multicolumn{6}{c}{Similarity and Paraphrase Tasks} \\ 
		\midrule
		MRPC & 3.7k & 1.7k & paraphrase &  acc./F1 & news \\
		STS-B & 7k & 1.4k & sentence similarity &  Pearson/Spearman corr. & misc. \\
		QQP & 364k & 391k & paraphrase &  acc./F1 & social QA questions \\
		\midrule
		\multicolumn{6}{c}{Inference Tasks} \\ 
		\midrule
		MNLI & 393k & 20k & NLI &  matched acc./mismatched acc. & misc. \\

		QNLI & 105k & 5.4k & QA/NLI &  acc. & Wikipedia \\
		
		RTE & 2.5k & 3k & NLI &  acc. & Wikipedia \\
		
		WNLI & 634 & 146 & coreference/NLI &  acc. & fiction books \\
		\bottomrule
	\end{tabular}
\end{table}

\begin{table}[h]
	\centering
	\caption{In supervised learning, the hyperparameters we used on the CIFAR dataset.
	}
	\label{app_table2}
	
	\begin{tabular}{|c|c|c|c|c|c|c|c|}
		\hline
		Epoch & lr &batch size&momentum&weight decay&optimizer&max lr&lr scheduler \\
		\hline
		200&0.2&128&0.9&5e-4&LARS(\cite{ref55})&5.2&OneCycle(\cite{ref56})\\
		\hline
	\end{tabular}
\end{table}
\begin{table}[h]
	\centering
	\caption{The hyperparameters we used in the ImageNet-1K experiments.}
	\label{app_table3}
	
	\begin{tabular}{|c|c|c|c|c|c|}
		\hline
		Epoch & lr &batch size&momentum&weight decay&optimizer \\
		\hline
		90&0.1&256&0.9&1e-4&SGD\\
		\hline
	\end{tabular}
\end{table}
\begin{table}[h]
	\centering
	\caption{The hyperparameters we used in the GLUE experiments.}
	\label{app_table4}
	\begin{tabular}{|c|c|c|c|c|}
		\hline
		Epoch & lr &batch size&optimizer &lr scheduler\\
		\hline
		10&2e-5&32& Adam(\cite{ref57})&linear\\
		\hline
	\end{tabular}
\end{table}
\begin{table}[h]
	\centering
	\caption{The hyperparameters we used in SimCLR.}
	\label{app_table5}
	\begin{tabular}{lllll}
		\hline
		\multicolumn{5}{c}{Pre-training} \\ 
		\hline
		Epoch & lr &batch size&optimizer &weight decay\\
		100&1e-3&256& Adam(\cite{ref57})&1e-6\\
		\hline
		\multicolumn{5}{c}{Fine-tuning} \\ 
		\hline
		Epoch & lr &batch size&optimizer &weight decay\\
		100&1e-3&32& Adam(\cite{ref57})&1e-6\\
		\hline
	\end{tabular}
\end{table}

\section{Implementation Details}\label{app_3}
In this section, we present the experimental details, including the selection of models and hyperparameters. In all experiments, we run using a single NVIDIA RTX 4090. It is worth noting, as mentioned in Section~\ref{4.3}, we keep the hyperparameter for the memory term, \( \beta \), consistently at 5, without the need for additional optimization.

\subsection{Supervised Learning}
In supervised learning, we conducted experiments on the CIFAR-10/100 and ImageNet-1K datasets. In CIFAR, we used the ResNet18 and ResNet34 models, along with the LARS optimizer \cite{ref55} and OneCycle learning rate schedule \cite{ref56}. For ImageNet-1K, we employed the ResNet50 model, and the learning rate was multiplied by 0.1 every 30 epochs. Our other hyperparameters are reported in Table~\ref{app_table2} and \ref{app_table3}.

In natural language understanding tasks, we used the pre-trained BERT-base model provided by Hugging Face. For all tasks in GLUE, we conducted experiments with uniform hyperparameters, which are specifically reported in Table~(\ref{app_table4}).

\subsection{Self-Supervised Learning}
In self-supervised learning, we used the SimCLR framework. In the image augmentation module, for both CIFAR-10 and CIFAR-100, we performed the following operations:
\begin{itemize} 
	\item Randomly cropped the image to a 32x32 region and resized it.
	\item Randomly horizontally flipped the image with a 50\% probability.
	\item Converted the image to grayscale with an 80\% probability.
	\item Normalized the image.
\end{itemize}
During the pre-training process, we added a linear head and used an exponential learning rate schedule. In the loss function, the temperature coefficient \( \tau \) was set to 0.5. The remaining hyperparameters are reported in Table~(\ref{app_table5}).
\end{appendices}

\newpage
\bibliography{sn-bibliography}% common bib file

\end{document}